\documentclass[letterpaper,11pt]{article}
\usepackage[margin=1in]{geometry}
\usepackage[bookmarks, colorlinks=true, plainpages = false, citecolor = blue,linkcolor=red,urlcolor = blue, filecolor = blue]{hyperref}
\usepackage{url}
\usepackage{amsmath,amsfonts,amsthm,amssymb,bm, verbatim,dsfont,mathtools}
\usepackage{color,graphicx,appendix}
\usepackage{subfigure}
\usepackage{etoolbox}
\usepackage{tikz}
\usepackage{xr,xspace}
\usepackage{todonotes}
\usepackage{paralist}
\usepackage{caption,soul}
\usepackage{algorithm}% http://ctan.org/pkg/algorithms
\usepackage{algorithmic}% http://ctan.org/pkg/algorithms

\newcommand{\post}[2]{\begin{center} \includegraphics[width=#2]{#1} \end{center} }
\newcommand\1[1]{\mathbb{I}_{\left\{#1\right\}}}
\theoremstyle{plain}
\newtheorem{theorem}{Theorem}
\newtheorem{lemma}{Lemma}

\newtheorem{corollary}{Corollary}
\theoremstyle{definition}
\newtheorem{definition}{Definition}

\newtheorem{conjecture}{Conjecture}

\newtheorem{remark}{Remark}
\newtheorem*{remark*}{Remark}

\newcommand{\argmax}{\mathop{\arg\max}}

\usepackage{xspace,prettyref}

\newcommand{\reals}{{\mathbb{R}}}
\newcommand{\integers}{{\mathbb{Z}}}

\newcommand{\eexp}{{\rm e}}

\newcommand{\identity}{\mathbf I}
\newcommand{\allones}{\mathbf J}

\newcommand{\diff}{{\rm d}}

\newcommand{\expect}[1]{\mathbb{E}\left[ #1 \right]}

\newcommand{\expects}[2]{\mathbb{E}_{#2}\left[ #1 \right]}

\newcommand{\prob}[1]{ \mathbb{P}\left\{ #1 \right\} }

\newcommand{\var}{\mathsf{var}}

\newcommand{\Bern}{{\rm Bern}}
\newcommand{\Binom}{{\rm Binom}}
\newcommand{\Pois}{{\rm Pois}}

\newcommand{\ie}{i.e.\xspace}

\newrefformat{eq}{(\ref{#1})}
\newrefformat{chap}{Chapter~\ref{#1}}
\newrefformat{sec}{Section~\ref{#1}}
\newrefformat{algo}{Algorithm~\ref{#1}}
\newrefformat{fig}{Fig.~\ref{#1}}
\newrefformat{tab}{Table~\ref{#1}}
\newrefformat{rmk}{Remark~\ref{#1}}
\newrefformat{clm}{Claim~\ref{#1}}
\newrefformat{def}{Definition~\ref{#1}}
\newrefformat{cor}{Corollary~\ref{#1}}
\newrefformat{lmm}{Lemma~\ref{#1}}
\newrefformat{prop}{Proposition~\ref{#1}}
\newrefformat{app}{Appendix~\ref{#1}}
\newrefformat{hyp}{Hypothesis~\ref{#1}}
\newrefformat{thm}{Theorem~\ref{#1}}

\newcommand{\lnorm}[2]{\left\|{#1} \right\|_{{#2}}}

\newcommand{\fnorm}[1]{\|#1\|_{\rm F}}
\newcommand{\Iprod}[2]{\langle #1, #2 \rangle}
\newcommand{\diag}[1]{\mathsf{diag} \left\{ {#1} \right\} }
\newcommand{\sign}{\mathsf{sign}}

\newcommand{\calL}{{\mathcal{L}}}

\newcommand{\calT}{{\mathcal{T}}}

\newcommand{\SDP}{{\rm SDP}\xspace}

\newcommand{\ER}{Erd\H{o}s-R\'enyi\xspace}

\renewcommand{\hat}{\widehat}
\renewcommand{\star}{\ast}

\begin{document}
%\bibliographystyle{imsart-number}

% "Title of the paper"
\title{Reconstruction in the Labeled Stochastic Block Model}
%\runtitle{Reconstruction in the Labeled Stochastic Block Model}
\date{\today}

\author{Marc Lelarge \and Laurent Massouli\'{e} \and Jiaming Xu \thanks{M. Lelarge is with INRIA-ENS, \texttt{marc.Lelarge@ens.fr}. L. Massouli\'{e}
is with the Microsoft Research-INRIA Joint Centre, \texttt{laurent.massoulie@inria.fr}.
J. Xu is with
 with the Department of ECE, University of Illinois at Urbana-Champaign, Urbana, IL, \texttt{jxu18@illinois.edu}. Part of  work is performed while
  J. Xu was an intern with Technicolor.  A preliminary version of this paper appeared in the Proceedings
of the 2013 Information Theory Workshop.}
 }

\maketitle
%\begin{aug}
%\author{\fnms{Marc} \snm{Lelarge} \thanksref{m1}\ead[label=e1]{Marc.Lelarge@ens.fr}},
%\author{\fnms{Laurent} \snm{Massouli\'{e}}\thanksref{m2}\ead[label=e2]{laurent.massoulie@inria.fr}}
%\and
%\author{\fnms{Jiaming} \snm{Xu}\thanksref{m3,t1}}
%\ead[label=e3]{jxu18@illinois.edu}
%
%\thankstext{t1}{Part of work performed while an intern with Technicolor}
%
%\runauthor{M. Lelarge et al.}
%
%\affiliation{INRIA-ENS \thanksmark{m1} and Microsoft Research-INRIA Joint Centre \thanksmark{m2} and  \\ University of Illinois at Urbana-Champaign \thanksmark{m3}}
%
%\address{Address of the First and Second authors\\
%Usually a few lines long\\
%\printead{e1}\\
%\phantom{E-mail:\ }\printead*{e2}}
%
%\address{Address of the Third author\\
%Usually a few lines long\\
%Usually a few lines long\\
%\printead{e3}\\
%\printead{u1}}
%\end{aug}

\begin{abstract}
The labeled stochastic block model is a random graph model representing networks with community structure and interactions of multiple types. In its simplest form, it consists of two communities of approximately equal size, and the edges are drawn and labeled at random with probability depending on whether their two endpoints belong to the same community or not.

It has been conjectured in \cite{Heimlicher12} that correlated reconstruction (i.e.\ identification of a partition correlated with the true partition into the underlying communities) would be feasible if and only if a model parameter exceeds a threshold.
We prove one half of this conjecture, i.e., reconstruction is impossible when below the threshold. In the positive direction, we introduce a weighted graph to exploit the label information. With a suitable choice of weight function, we show that when above the threshold by a specific constant, reconstruction is achieved by (1) minimum bisection, (2) a semidefinite relaxation of minimum bisection, and (3) a spectral method combined with removal of edges incident to vertices of high degree. 
Furthermore, we show that  hypothesis testing between the labeled stochastic block model and the labeled Erd\H{o}s-R\'enyi random graph model exhibits a phase transition at the conjectured reconstruction threshold. 
%with the same threshold as the correlated reconstruction.

\end{abstract}

%\begin{keyword}[class=AMS]
%\kwd[Primary ]{}
%\kwd{}
%\kwd[; secondary ]{}
%\end{keyword}

%\begin{keyword}
%\kwd{}
%\kwd{}
%\end{keyword}

%\end{frontmatter}

\section{Introduction}
\subsection{Motivation}
Community detection aims to identify underlying communities of similar characteristics in an overall population from the observation of pairwise interactions between individuals \cite{Fortunato10,Newman04,Newman06}. The stochastic block model, also known as {\it planted partition model}, is a popular random graph model for analyzing the community detection problem \cite{Holland83,Snijders97,Bicke09,Yu11,Decelle11}, in which pairwise interactions are binary: an edge is either present or absent between two individuals. In its simplest form, the stochastic block model consists of two communities of approximately equal size, where the within-community edge is present at random with probability $p$; while the across-community edge is present with probability $q$. If $p>q$, it corresponds to assortative communities where interactions are more likely within rather than across communities; while $p<q$ corresponds to disassortative communities.

In practice, interactions can be of various types and these types reveal more information on the underlying communities than the mere existence of the interaction itself. For example, in recommender systems, interactions between users and items come with user ratings. Such ratings contain far more information than the interaction itself to characterize the user and item types. Similarly, protein-protein chemical interactions in biological networks can be exothermic and endothermic; email exchanges in a club may be formal or informal; friendship in social networks may be strong or weak. The labeled stochastic block model was recently proposed in \cite{Heimlicher12} to capture rich interaction types. In this model interaction types are described by labels drawn from an arbitrary collection. In particular, for the simple two communities case, the within-community edge is labeled at random with distribution $\mu$; while the across-community edge is labeled with a different distribution $\nu$. In this context an important question is how to leverage the labeling information for detecting underlying communities.

\subsection{Information-Scarce Regime}
In this paper, we focus on the sparse labeled stochastic block model in which every vertex has a limited average degree, i.e., $p,q=O(1/n)$, where $n$ is the number of vertices. It corresponds to the information-scarce regime where only $O(n)$ edges and labels are observed in total\footnote{We also provide results for $p,q=O(\hbox{polylog}(n)/n)$ in Theorem \ref{ThmSpectralSBM-large}.}. This regime is of practical interest, arising in several contexts. For example, in recommender systems, users only give ratings to few items; in biological networks, only few protein-protein interactions are observed due to cost constraints; in social networks, a person only has a limited number of friends.

For the stochastic block model in this information-scarce regime, there are $\Theta(n)$ isolated vertices, as in Erd\H{o}s-R\'enyi random graphs with bounded average degree. For isolated vertices, it is impossible to determine their community membership and thus exact reconstruction of communities is impossible. Therefore, we resort to finding a partition into communities positively correlated to the true community partition (see Definition \ref{def:Q} below).

\subsection{Main Results}
Focusing on the two communities scenario, we show that a positively correlated reconstruction is fundamentally impossible when below a threshold. This establishes one half of the conjecture in \cite{Heimlicher12}. In the positive direction, we establish the following results. We introduce a graph weighted by a suitable function of observed labels, on which we show that:

(1) Minimum bisection gives a positively correlated partition when above the threshold by a factor of $64 \ln 2$.

(2) A semidefinite relaxation of minimum bisection gives a positively correlated partition when above the threshold by a factor of $2^{17} \ln 2$.

(3) A spectral method combined with removal of edges incident to vertices of high degree gives a positively correlated partition when above the threshold by a constant factor.

Furthermore, we show that the labeled stochastic block model is contiguous to a labeled Erd\H{o}s-R\'enyi random graph when below the reconstruction threshold and orthogonal to it when above the threshold. It implies that for the hypothesis testing problem between the labeled stochastic block model and the labeled Erd\H{o}s-R\'enyi random graph model, the correct identification of the underlying distribution is feasible if and only if above the reconstruction threshold. It also implies that there is no consistent estimator for model parameters when below the reconstruction threshold.

\subsection{Related Work}
For the stochastic block model, most previous work focuses on the ``dense'' regime with an average degree  diverging as the size of the graph $n$ grows, (see, e.g., \cite{Chen12,ChenXu14} and the references therein).  %McSherry \cite{McSherry01} showed that the spectral method works as long as $p-q \ge \Omega (\sqrt{p \log n /n })$ with average degree as low as $\Omega(\log^6 n)$. Massouli\'e and Tomozei~\cite{mastom11} reduced this lower bound on the average degree to $\Omega(\log n)$. More recently, it was shown in \cite{Chen12} that a matrix completion approach works  when  $p-q \ge \Omega (r \sqrt{p/n} \log^2 n )$  where the number of communities $r$ could scale with $n$.
For the ``sparse'' regime with bounded average degrees, a sharp phase transition threshold for reconstruction was conjectured in \cite{Decelle11}  by analyzing the belief propagation algorithm. The converse part of the conjecture was rigorously proved in \cite{Mossel12}. The achievability part is proved independently in \cite{Mossel13,Massoulie13}. In addition, it is shown in \cite{Coja-oghlan10} that a variant of spectral method gives a positively correlated partition when above the threshold by an unknown constant factor. More recently, it is shown in \cite{Vershynin14} that a semidefinite program finds a correlated partition when above the threshold by some large constant factor.
%In the converse direction,

The labeled stochastic block was first proposed and studied in \cite{Heimlicher12} and a new reconstruction threshold that incorporates the extra labeling information was conjectured. Simulations further indicate that the belief propagation algorithm works when above the threshold, but reconstruction algorithms that provably work are still unknown. 

Finally, we recently became aware of the work \cite{ABBS14} that studies the problem of decoding 
binary node labels from noisy edge measurements. In the case where the background graph is \ER random graph and each node label is independently and uniformly chosen from $\{\pm 1\}$, 
the model in \cite{ABBS14} can be viewed as a special case
of the labeled stochastic block model with $p=q$, $\mu=(1-\epsilon) \delta_{+1} + \epsilon \delta_{-1}$ and $\nu=\epsilon \delta_{+1} + (1-\epsilon) \delta_{-1}$, where $\delta_{x}$ 
denotes the probability measure concentrated on point $x$ (See Section \ref{SectionModel} for the formal model description).  When $p =q = a \log n /n$ for some constant $a$ and $\epsilon \to 1/2$, it is shown in \cite{ABBS14} that exact recovery of node labels is
possible if and only if $a (1-2\epsilon)^2 >2$. In contrast, our results show that when $p=q =a /n$ for some constant $a$, correlated recovery of node labels is impossible
if $a (1-2 \epsilon)^2 <1$ for any $0\le \epsilon \le 1$.  Moreover, we show that distinguishing hypothesis $\epsilon =\epsilon_0$ and hypothesis $\epsilon = 1/2$ is possible if and only if $ a(1-2\epsilon_0)^2>1$. 

\subsection{Outline}
Section \ref{SectionModel} introduces the precise definition of the labeled stochastic block model to be studied and the key notations. The main theorems are introduced and briefly discussed in Section \ref{SectionMainThm}. The detailed proofs are presented in Section \ref{SectionProof}. Section \ref{SectionConclusion} ends the paper with concluding remarks. Miscellaneous details and proofs are in the Appendix.

\section{Model and Notation} \label{SectionModel}
This section formally defines the labeled stochastic block model with two symmetric communities and introduces the key notations and definitions used in the paper. Let $\mathcal{L}$ denote a finite set. The labeled stochastic block model $\mathcal{G}(n,p,q,\mu,\nu)$ is a random graph with $n$ vertices of $\{\pm 1\}$ types indexed by $[n]$ and $\{\ell \in \mathcal{L} \}$-labeled edges. To generate a particular realization $(G,L,\sigma)$, first assign type $\sigma_u \in \{\pm 1\}$ to each vertex $u$ uniformly and independently at random. Then, for every vertex pair $(u,v)$, independently of everything else, draw an edge between $u$ and $v$ with probability $p$ if $\sigma_u=\sigma_v$ and with probability $q$ otherwise. Finally, every edge $e=(u,v)$ is labeled with $\ell$ independently at random with probability $\mu(\ell)$ if $\sigma_u=\sigma_v$ and with probability $\nu(\ell)$ otherwise.

Equivalently, we can specify $\mathcal{G}(n,p,q,\mu,\nu)$ by its probability distribution. Let
\begin{align}
\phi_{uv}(G,L,\sigma)=\left \{
\begin{array}{rl}
 p \mu(L_{uv}) & \text{if } \sigma_u=\sigma_v, (u,v) \in E(G), \\
 q \nu(L_{uv}) & \text{if } \sigma_u \neq \sigma_v, (u,v) \in E(G), \\
 1-p & \text{if } \sigma_u= \sigma_v, (u,v) \notin E(G), \\
 1-q & \text{if } \sigma_u \neq \sigma_v, (u,v) \notin E(G), \nonumber
\end{array} \right.
\end{align}
where  $E(G)$ is the set of edges of $G$ and $L_{uv}$ is the label on the edge $(u,v)$. Then,
\begin{align}
\mathbb{P}_n (G, L, \sigma) = 2^{-n} \prod_{(u,v): u<v} \phi_{uv}(G,L,\sigma). \label{eq:loglikelihood}
\end{align}
When $\mu=\nu$, it reduces to the classical stochastic block model without labels. This paper focuses on the sparse case where $p=a/n$ and $q=b/n$ for two fixed constants $a$ and $b$, and the goal is to reconstruct the true underlying types of vertices $\sigma$ by observing the graph structure $G$ and the labels on edges $L$.

It is known that in the sparse graph, there are $\Theta(n)$ isolated vertices whose types clearly cannot be recovered accurately. Therefore, our goal is to reconstruct a type assignment which is positively correlated to the true type assignment. More formally, we adopt the following definition.
\begin{definition}\label{def:Q}
A type assignment $\hat{\sigma}$ is said to be positively correlated with the true type assignment $\sigma$ if a.a.s.
\begin{align}
\ Q(\sigma,\hat{\sigma}) := \frac{1}{2} - \frac{1}{n} \min \{ d (\sigma, \hat{\sigma} ), d(\sigma, -\hat{\sigma}) \} >0,
\end{align}
where $d$ is the Hamming distance, and $Q$ is called the {\it Overlap}.
\end{definition}
The shorthand a.a.s. denotes {\it asymptotically almost surely}. A sequence of events $A_n$ holds a.a.s. if the probability of $A_n$ converges to $1$ as $n \to \infty$.
Define $\tau$ as
\begin{align}
\tau= \frac{a+b}{2} \sum_{ \ell \in \mathcal{L} } \frac{a\mu(\ell) + b\nu(\ell) }{a+b}  \left( \frac{a\mu(\ell)-b\nu(\ell)}{a\mu(\ell)+b \nu(\ell) } \right)^2. \label{DefReconstructionThreshold}
\end{align}
It was conjectured in \cite{Heimlicher12} that $\tau$ is the threshold for positively correlated reconstruction.
\begin{conjecture} \label{Conjecture}
\begin{itemize}
\item[(i)] If $\tau>1$, then it is possible to find a type assignment correlated
with the true assignement a.a.s.
\item[(ii)] If $\tau<1$, then it is impossible to find a type assignment correlated
with the true assignement a.a.s.
\end{itemize}
\end{conjecture}
In this paper, we prove (ii) and propose three different algorithms
able to find a type assignment correlated with the true assignment
for $\tau$ big enough.

\paragraph{Notation}
Let $A$ denote the adjacency matrix of the graph $G$, $\identity $ denote the identity matrix,
and $\allones$ denote the all-one matrix.
We write  $X \succeq 0$ if $X$ is positive semidefinite and $X \ge 0$ if all the entries of $X$ are non-negative.
%Let $\calS^n$ denote the set of all $n \times n$ symmetric matrices.
%For $X \in \calS^n$, let $\lambda_2(X)$ denote its second smallest eigenvalue.
For any matrix $Y$, let $\|Y\|$ denote its spectral norm.
For any positive integer $n$, let $[n]=\{1, \ldots, n\}$.
For any set $T \subset [n]$, let $|T|$ denote its cardinality and $T^c$ denote its complement.
We use standard big $O$ notations,
e.g., for any sequences $\{a_n\}$ and $\{b_n\}$, $a_n=\Theta(b_n)$ or $a_n  \asymp b_n$
if there is an absolute constant $c>0$ such that $1/c\le a_n/ b_n \le c$.
Let $\Bern(p)$ denote the Bernoulli distribution with mean $p$ and
$\Binom(N,p)$ denote the binomial distribution with $N$ trials and success probability $p$.
All logarithms are natural and we use the convention $0 \log 0=0$. For a vector $x \in \reals^n$, $\sign(x)$ gives the sign of
$x$ componentwise, and $\|x\|$ denotes the $L_2$ norm. For a graph $G$, let $V(G)$ denote its vertex set and $E(G)$ 
denote its edge set. 

\section{Main Theorems} \label{SectionMainThm}
%This section presents the main theorems in this paper.

\subsection{Minimum Bisection}

To recover the community partition, one approach is via the maximum
likelihood estimation. In view of \prettyref{eq:loglikelihood}, the log-likelihood function can be written as:
\begin{eqnarray*}
\log \mathbb{P}( G,L | \sigma)&=& \frac{1}{2} \sum_{(u,v) \in E(G) } \left[ \log
\frac{a\mu(L_{uv})}{b\nu(L_{uv})}  \sigma_u \sigma_v  + \log \left( \frac{ab}{n^2}  \mu (L_{uv})\nu(L_{uv}) \right) \right] \\
&+&\frac{1}{2} \sum_{(u,v) \notin E(G) } \left[  \log\left(\frac{1-a/n}{1-b/n}\right)\sigma_u \sigma_v  + \log \left( (1-a/n)(1-b/n) \right) \right].
\end{eqnarray*}
Under the constraint $\sum_{u} \sigma_u=0$, the maximum likelihood estimation is equivalent to
 \begin{align}
\max_{\sigma}  \quad & \sum_{(u,v)\in E(G) } \log \left[
  \frac{a (1-b/n)\mu(L_{uv}) }{b (1-a/n) \nu(L_{uv})}  \right] A_{uv} \sigma_u \sigma_v  \nonumber \\
\text{s.t. }  \quad &  \sum_u \sigma_u =0, \;  \sigma \in \{ \pm 1 \}^n \nonumber .
\end{align}
This is equivalent to the minimum bisection on the weighted graph with a specific weight function $w(\ell) = \log \frac{a (1-b/n) \mu(\ell) }{b (1-a/n) \nu(\ell) }$. For a general weighing function $w:\mathcal{L} \to [-1,1]$, the minimum bisection finds a balanced bipartite subgraph in $G$ with the minimum weighted cut, i.e.,
\begin{align}
\min_{\sigma}  & \sum_{(u,v):\sigma_u \neq \sigma_v}  W_{uv}  \nonumber \\
\text{s.t. 	} & \sum_{u} \sigma_u =0, \; \sigma_u \in \{ \pm 1\}, \label{eq:MinimumBisection}
\end{align}
where $W_{uv}=A_{uv} w(L_{uv})$ and $A$ is the adjacency matrix of $G$.

\begin{theorem} \label{ThmMinBisection}
Assume the technical condition: $\sum_\ell a \mu(\ell) w^2(\ell),
\sum_\ell b \nu(\ell) w^2(\ell) > 8 \ln 2$. Then if
\begin{align}
\frac{\sum_\ell (a\mu(\ell)-b\nu(\ell) )w (\ell)}{\sqrt{ \sum_\ell (a\mu(\ell)+b \nu(\ell)) w^2(\ell) }} > \sqrt{128 \ln 2 }  \label{EqMinBisectionCondition},
\end{align}
a.a.s. solutions of the minimum bisection (\ref{eq:MinimumBisection}) are
positively correlated to the true type assignment $\sigma^\ast.$
Moreover, the left hand side of (\ref{EqMinBisectionCondition}) is
maximized when $w(\ell)=\frac{a\mu(\ell)-b\nu(\ell)}{a\mu(\ell)+
  b\nu(\ell) }$, in which case (\ref{EqMinBisectionCondition}) reduces
to $\tau > 64\ln 2$.
\end{theorem}

\subsection{Semidefinite relaxation method}\label{SectionSDP}
The minimum bisection is known to be NP-hard in the worst case \cite[Theorem 1.3]{garey76}. In this section,
we present a semidefinite relaxation of the minimum bisection \prettyref{eq:MinimumBisection} which
is solvable in polynomial time, and show it finds an assignment correlated with the true assignment provided
$\tau$ is large enough.
Let $Y=\sigma \sigma^\top$. Then $\sigma_u = \pm 1$  is equivalent to $Y_{uu}=1$, and $\sum_{u} \sigma_u =0$
if and only if $\Iprod{Y}{\allones}=0$. Therefore, \prettyref{eq:MinimumBisection} can be recast as
\begin{align}
\max_{Y,\sigma}  & \; \Iprod{W}{Y} \nonumber  \\
\text{s.t.	} & \; Y=\sigma \sigma^\top  \nonumber  \\
& \;  Y_{uu} =1, \quad u \in [n]\nonumber \\
& \;  \Iprod{\allones}{Y} =0 . \label{eq:SBMMB2}
\end{align}
Notice that the matrix $Y=\sigma \sigma^\top$ is a rank-one positive semidefinite matrix. If we relax this
condition by dropping the rank-one restriction, we obtain the following semidefinite relaxation of \prettyref{eq:SBMMB2}:
%\begin{algorithm}[H]
%\caption{Convex relaxation of $\ML$ estimation \label{alg:convex}}
\begin{align}
\widehat{Y}_{\SDP} = \argmax_{Y}  & \; \langle W, Y \rangle \nonumber  \\
\text{s.t.	} & \; Y \succeq 0  \nonumber \\
 & \;  Y_{uu} =1, \quad u \in [n] \nonumber \\
 & \; \Iprod{\allones}{Y} =0. \label{eq:SBMSDP}
\end{align}
To get an estimator of the type assignment from $\widehat{Y}_{\SDP}$, let $y$ denote an
eigenvector of $\widehat{Y}_{\SDP}$ corresponding to the largest eigenvalue and $\|y\| =\sqrt{n}$.
The following result shows that $\widehat{\sigma}_{\SDP} \triangleq \sign(y)$ is positively correlated with the true type assignment.
\begin{theorem}\label{thm:SBMSDPCorrelated}
Assume the technical condition: $\sum_{\ell} w^2(\ell) ( a \mu(\ell) +b \nu(\ell) ) > 8 \ln 2$. If
\begin{align}
\frac{ \sum_\ell (a\mu(\ell)-b\nu(\ell) )w (\ell)  } { \sqrt{\sum_\ell (a\mu(\ell)+b \nu(\ell)) w^2(\ell)} } > 512 \sqrt{ \ln 2} \label{eq:SBMSDPCondition},
\end{align}
then a.a.s.\ $\widehat{\sigma}_{\SDP}$ is positively correlated to the true type assignment $\sigma^\star$.
Moreover, the left hand side of \prettyref{eq:SBMSDPCondition} is
maximized when $w(\ell)=\frac{a\mu(\ell)-b\nu(\ell)}{a\mu(\ell)+
  b\nu(\ell) }$, in which case (\ref{eq:SBMSDPCondition}) reduces
to $\tau > 2^{17} \ln 2$.
\end{theorem}
In the stochastic block model without labels, \ie, $\mu=\nu$, condition (\ref{eq:SBMSDPCondition}) reduces to $(a-b)^2> 2^{18} \ln 2 (a+b)$;
similar conditions with a different constant have been proved in \cite[Theorem 1.1]{Vershynin14} using the Grothendieck's inequality. Our proof builds upon the
analysis in \cite{Vershynin14}.
\subsection{Spectral Method}

In this section, we present a polynomial-time spectral algorithm based on the weighted adjacency matrix
$W$ and show that this algorithm allows us to
find an assignment correlated with the true assignment provided
$\tau$ is large enough.

Note that $\mathbb{E}[W | \sigma ] =\frac{\alpha}{n}  \allones+ \frac{\beta}{n} \sigma \sigma ^\top- \frac{\alpha+\beta}{n} \mathbf{I}$ with
\begin{align}
\alpha=\frac{1}{2} \sum_{\ell} w(\ell) (a \mu(\ell) + b \nu(\ell)), \nonumber \\
\beta=\frac{1}{2} \sum_{\ell} w(\ell) (a \mu(\ell) -b \nu(\ell)). \label{EqDefAlphaBeta}
\end{align}
The term $\frac{\alpha+\beta}{n} \mathbf{I}$ is irrelevant to the main
results (thanks to Weyl's perturbation theorem) and neglected for simplicity.
Let $D=W-\frac{\alpha}{n}  \allones$ and then $\mathbb{E}[D | \sigma ]= \frac{\beta}{n} \sigma \sigma ^\top$ has rank one with singular value $\beta$. Hence, it makes sense to define $\hat{D}$ as the best rank-1 approximation of the
matrix $D$. In other words, if $D=\sum_{i}v_ix_ix_i^\top$ is the
eigenvalue decomposition of $D$ with eigenvalues $|v_1| \geq | v_2 | \geq \dots$, we
define $\hat{D} = v_1x_1x_1^\top$. Then if the matrix $D$ is close
to its mean $\mathbb{E}[D | \sigma ]$ in the spectral norm, we expect $v_1$ to be close to $\beta$, and $\sign (x_1)$ to
be correlated with $\sigma$.
%Note that $D$ is very similar to the modularity matrix defined in \cite{Newman06} and thus dividing the vertices into two communities according to ${\rm sign} (x_1)$ can be seen as an algorithm to maximize the modularity.
Unfortunately, in the sparse regime, there are
vertices of degree $\Omega(\frac{\log n} {\log \log n})$ and thus the
largest singular value of $W$ could reach $\Omega(\sqrt{\frac{\log n}
  { \log \log n}})$ which is much higher than $\beta$.
%For the regular graph, such issue does not exist. In particular, consider a regular version of the labeled stochastic block model. Every vertex is assigned a type from $\{\pm 1 \}$ arbitrarily and the graph $G$ is generated by a random regular-$a$ graph model independently from the type assignment. Every edge $e=(u,v)$ is labeled with $\ell \in \mathcal{L}$ with probability $\mu(\ell)$ if $\sigma_u=\sigma_v$ and with probability $\nu(\ell)$ otherwise.
%\begin{theorem} \label{ThmSpectralRegular}
%Let $w(\ell)=\frac{\mu(\ell)-\nu(\ell)}{\mu(\ell)+\nu(\ell)}$. For the regular version of the labeled stochastic block model, there exists a constant $C(\mu,\nu)$ which only depends on $\mu$ and $\mu$ such that if $\tau>C(\mu,\nu)$, then the eigenvector corresponding to the largest eigenvalue of $W$ is positively correlated with the true type assignment a.s.s..
%\end{theorem}
In order to take care of the issue, we begin with a preliminary step to clean the spectrum of $W$:
we remove all edges incident to vertices in the graph with degree larger than
$\frac{3}{2}\frac{a+b}{2}$. To summarize, for a given weight function
$w(\ell)$, our algorithm $\rm{Spectral-Reconstruction}$ has the following structure:
\begin{enumerate}
%\item For vertices with degree larger than $\frac{3}{2}\frac{a+b}{2}$
%  and assign a random type to these vertices.
\item Remove edges incident to vertices with degree larger than $\frac{3}{2}\frac{a+b}{2}$ and let $G'$ denote
the resulting graph. Define $W'$ to be the weighted adjacency matrix of $G'$.
%Define $W^\prime$ by setting to zero the rows and columns of $W$ corresponding to vertices removed.
\item Let $\hat{x}$ be the left-singular vector associated with
  the largest singular value of $D^\prime=W^\prime-\frac{\alpha}{n}  \allones $, i.e.,
\begin{align}
\hat{x} = \arg \max\{ | x^\top D^\prime x |  ,\: \|x\|=1 \}. \label{eq:spectral}
\end{align}
Output ${\rm sign} (\hat{x})$ for the types of the vertices.
\end{enumerate}
Observe that \prettyref{eq:spectral} can be seen as a (non-convex) relaxation of the minimum bisection (\ref{eq:MinimumBisection}) by replacing the integer constraint with the unit-norm constraint and relaxing the constraint $\sum_{u} \sigma_u=0$ to be a regularized term $\frac{\alpha}{n} x^\top \allones x $ in the objective function. $\rm{Spectral-Reconstruction}$ needs estimates of $\alpha$ and $a+b$, which can be well approximated by $ \frac{1}{n} \mathbf{1}^\top W \mathbf{1}$ and $ \frac{2}{n} \mathbf{1}^\top A \mathbf{1}$, respectively. To simplify the analysis, we will assume that the exact values of $\alpha$ and $a+b$ are known.

\begin{theorem}\label{ThmSpectralSBM}
Assume $a>b>C_0$ for some sufficiently large constant $C_0$. There exists a universal constant $C$ (i.e.\ not depending on $a$, $b$,
$\mu$ or $\nu$) such that if $\beta^2 >C(a+b)$, where $\beta$
is defined in \eqref{EqDefAlphaBeta},  then a.a.s.\ $\rm{Spectral-Reconstruction}$
outputs a type assignment correlated with the true assignment.
In the particular case, where $w(\ell) = \frac{a\mu(\ell)-b\nu(\ell)}{a\mu(\ell)+
  b\nu(\ell) }$, the condition $\beta^2>C(a+b)$ reduces to $\tau> \sqrt{C(a+b) }$.
\end{theorem}
In the stochastic block model without labels, letting $w(\ell)=1$, condition $\beta^2>C(a+b)$ reduces to $(a-b)^2> 4 C(a+b)$;
the sharp condition $(a-b)^2> 2(a+b)$ has been proved recently in \cite{Mossel13,Massoulie13}.
Compared to point (i) in the Conjecture \ref{Conjecture}, our result
does not give the right order of magnitude when $a$ and $b$ are large. Indeed, we are able to
improve it if we allow $a$ and $b$ to grow with $n$.

\begin{theorem}\label{ThmSpectralSBM-large}
Assume that $\min(a,b) =\Omega(\log^6n)$.
If
\begin{align}
\frac{ [\sum_\ell (a\mu(\ell)-b\nu(\ell) )w (\ell) ]^2 } { \sum_\ell
  (a\mu(\ell)+b \nu(\ell)) w^2(\ell) } > 256 \label{EqSpecCondition},
\end{align}
then $\rm{Spectral-Reconstruction}$ %leaving out step 1)
outputs a type assignment correlated with the true assignment a.a.s.
Moreover, the left hand side of (\ref{EqSpecCondition}) is
maximized when $w(\ell)=\frac{a\mu(\ell)-b\nu(\ell)}{a\mu(\ell)+
  b\nu(\ell) }$, in which case (\ref{EqSpecCondition}) reduces
to $\tau > 128$. With this choice of $w(\ell)$, as soon as $\tau\to
\infty$, $\rm{Spectral-Reconstruction}$ %leaving out step 1)
outputs the true assignment
for all vertices except $o(n)$ a.a.s.
\end{theorem}
Note that in the regime $\min(a,b) =\Omega(\log^6n)$, the degrees are very concentrated and step 1) of the algorithm can be removed without harm.
The simulation results, depicted in Fig.~\ref{FigSBMSpectralMethod}, further indicate that $\rm{Spectral-Reconstruction}$ leaving out step 1) outputs a positively correlated assignment when above the threshold. In the simulation, we assume for simplicity only two labels: $r$ and $b$, and define $\mu(r)=0.5+\epsilon$ and $\nu(r)=0.5-\epsilon$. We generate the graph from the labeled stochastic block model with $n=1000$ vertices for various $a,b,\epsilon$. Fix $a,b$, we plot the overlap $Q$ against $\epsilon$ and indicate the threshold $\tau=1$ as a vertical dash line. All plotted values are averages over $100$ trials.

\begin{figure}
\centering
\post{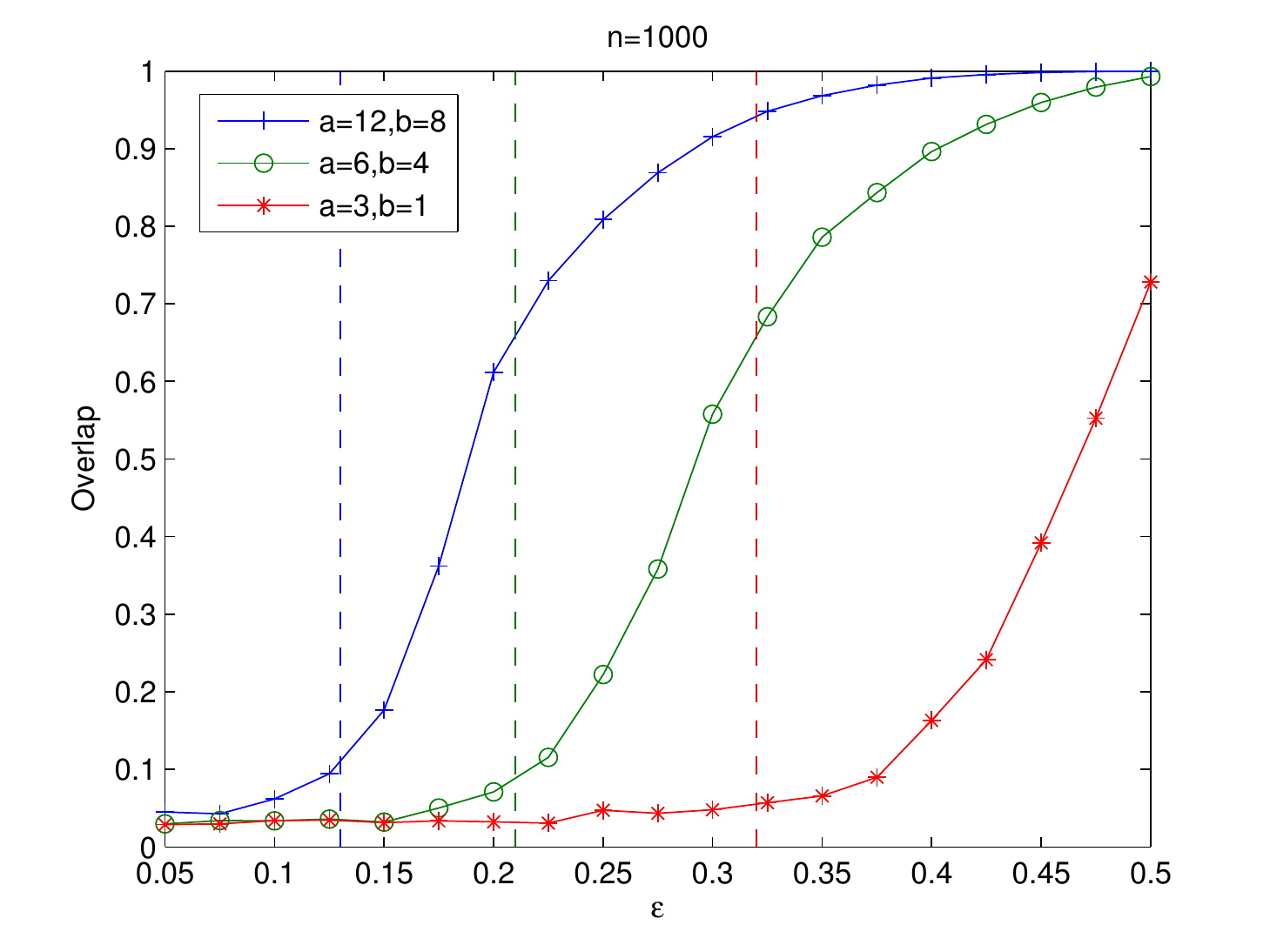}{3.5in}
\centering
\caption{The overlap $Q$ against $\epsilon$ from $0.05$ to $0.5$.}
\label{FigSBMSpectralMethod}
\end{figure}

Note that our algorithm is most efficient when the parameters ($a$,
$b$, $\mu$ and $\nu$) of the model are known as the optimal weight
function depends on these parameters. In the case where the labels are
uninformative, i.e. $\mu=\nu$, our algorithm is very simple, does not
require to know the values $a$ and $b$, and in the range of Theorem
\ref{ThmSpectralSBM-large}, has the best known performance
guarantee (see \cite[Table I]{Chen12}).

\subsection{Converse Result}
This section proves part (ii) of Conjecture \ref{Conjecture}. In particular, we show that when $\tau<1$, asymptotically it is impossible to tell whether any two vertices are more likely to belong to the same community. It further implies that reconstructing a positively correlated type assignment is fundamentally impossible.
\begin{theorem} \label{ThmNonReconstruction}
If $\tau<1$, then for any fixed vertices $\rho$ and $v$,
\begin{align}
\mathbb{P}_n (\sigma_\rho=+1 | G, L, \sigma_v=+1) \to 1/2 \text{ a.a.s}.
\end{align}
\end{theorem}
\begin{remark}
Reconstructing a positively correlated type assignment is harder than telling whether any two vertices are more likely to belong to the same community. In particular, given a positively correlated type assignment $\hat{\sigma}$, for two vertices randomly chosen, they are more likely to belong to the same community if they have the same type in $\hat{\sigma}$.
\end{remark}
Theorem \ref{ThmNonReconstruction} is related to the Ising spin model in the statistical physics \cite{Peres00,Mossel04}, and it essentially says that there is no long range correlation in the type assignment when $\tau<1$. The main idea in the proof of Theorem \ref{ThmNonReconstruction} is borrowed from \cite{Mossel12} and works as follows: (1) pick any two fixed vertices $\rho,v$ and consider the local neighborhood of $\rho$ up to distance $O(\log (n))$. The vertex $v$ lies outside of the local neighborhood of $\rho$ a.a.s.. (2) conditional on the type assignment at the boundary of the local neighborhood, $\sigma_\rho$ is asymptotically independent with $\sigma_v$. (3) the local neighborhood of $\rho$ looks like a Markov process on a labeled Galton-Watson tree rooted at $\rho$. (4) For the Markov process on the labeled Galton-Watson tree, the types of leaves provide no information about the type of the root $\rho$ when the depth of tree goes to infinity.

\subsection{Hypothesis Testing}
Consider a labeled Erd\H{o}s-R\'enyi random graph $\mathcal{G}(n,\frac{a+b}{2})$, where independently at random, each pair of two vertices is connected with probability $\frac{a+b}{2}$, and every edge is labeled with $\ell \in \mathcal{L}$ with probability $\frac{a\mu(\ell)+b\nu(\ell)}{a+b}.$  Let $\mathbb{P}^\prime_n$ denote the distribution of the labeled Erd\H{o}s-R\'enyi random graph.

Given a graph $(G,L)$ which was drawn from either $\mathbb{P}_n$ or $\mathbb{P}^\prime_n$, an interesting hypothesis testing problem is to decide which one is the underlying distribution of $(G,L)$? It turns out that when $\tau>1$, the correct identification of the underlying distribution is feasible a.a.s.; however, when $\tau<1$, one is bound to make error with non-vanishing probability.

\begin{theorem} \label{ThmACER}
If $\tau>1$, then $\mathbb{P}_n$ and $\mathbb{P}^\prime_n$ are asymptotically orthogonal, i.e., there exists event $A_n$ such that $\mathbb{P}_n(A_n) \to 1 $ and $\mathbb{P}^\prime_n (A_n) \to 0$.

If $\tau<1$, then $\mathbb{P}_n$ and $\mathbb{P}^\prime_n$ are contiguous, i.e., for every sequence of event $A_n $,
\begin{align}
\lim_{n \to \infty} \mathbb{P}_n(A_n)=0 \Leftrightarrow \lim_{n \to \infty} \mathbb{P}_n^\prime(A_n)=0. \nonumber
\end{align}
\end{theorem}

Theorem \ref{ThmACER} further implies the following corollary regarding the model parameter estimation.
\begin{corollary}
If $\tau<1$, then there is no consistent estimator for parameters $a,b,\mu,\nu$.
\end{corollary}
\begin{proof}
The second part of Theorem \ref{ThmACER} implies that $\mathcal{G}(n,\frac{a_1}{n},\frac{b_1}{n},\mu_1,\nu_1)$ and $\mathcal{G}(n,\frac{a_2}{n},\frac{b_2}{n},\mu_2,\nu_2)$ are contiguous as long as $a_1 \mu_1(\ell)+b_1\nu_1(\ell)=a_2\mu_2(\ell)+b_2\nu_2(\ell)$ and
\begin{align}
\sum_\ell \frac{ (a_i \mu_i(\ell)-b_i\nu_i(\ell))^2} {2(a_i \mu_i(\ell) + b_i \nu_i(\ell) ) }<1, \nonumber
\end{align}
for $i=1,2$.
Therefore, one cannot distinguish between $\mathcal{G}(n,\frac{a_1}{n},\frac{b_1}{n},\mu_1,\nu_1)$ and $\mathcal{G}(n,\frac{a_2}{n},\frac{b_2}{n},\mu_2,\nu_2)$ with the success probability converging to $1$, and thus there is no consistent estimator for parameters $a,b,\mu,\nu$.
\end{proof}
 In the special case where $\mu=\nu$, i.e., no labeling information is available, Theorem \ref{ThmACER} reduces to Theorem 2.4 in \cite{Mossel12}. The positive part of Theorem \ref{ThmACER} is proved by counting the number of labeled short cycles and the second moment method. The negative part of Theorem \ref{ThmACER} is proved using the small subgraph conditioning method as introduced in \cite{Mossel12}. The small subgraph conditioning method was originally developed to show that random $d$-regular graphs are Hamiltonian a.s.s. \cite{Wormald94, Janson11}.

\section{Proofs} \label{SectionProof}
\subsection{Proof of Theorem \ref{ThmMinBisection}} \label{SectionMinBisectionpf}
Recall that $\sigma^\ast$ denotes the true type assignment. Since $| \{ u: \sigma^\ast_u = 1\}| \sim \Binom(n,1/2)$, by Chernoff bound, a.a.s.,
\begin{align}
|\{ u: \sigma^\ast_u = 1\}| \in \left[ n/2-\sqrt{n \log n}, n/2+ \sqrt{n \log n} \right]. \label{eq:clustersizebalance}
\end{align}
For ease of presentation, assume $|\{ u: \sigma^\ast_u = 1\}|=n/2$.
Let $m(\sigma) \triangleq |\{u: \sigma_u=+1, \sigma^\star_u=-1 \} |$ and $\epsilon>0$ be an arbitrarily small constant. To prove the theorem, by the definition of positively correlated reconstruction, it suffices to show that for all $\sigma$ with $\frac{n}{4} (1-\epsilon) \le m(\sigma) \le \frac{n}{4}$,
\begin{align}
\sum_{ \substack{ (u,v):\sigma_u \neq \sigma_v, \\ \sigma_u^\star=\sigma_v^\star }} W_{uv} - \sum_{\substack{ (u,v):\sigma_u=\sigma_v, \\ \sigma_u^\star \neq \sigma_v^\star} } W_{uv} :=Y_1(\sigma)-Y_2(\sigma) > 0. \nonumber
\end{align}
To ease the notation, we suppress the argument $\sigma$. Observe that $Y_1$ is a sum of $2m(n/2-m)$ i.i.d. random variables whose value is $w(\ell)$ with probability $\frac{a}{n} \mu(\ell)$; $Y_2$ is a sum of $2m(n/2-m)$ i.i.d. random variables whose value is $w(\ell )$ with probability $\frac{b}{n} \nu(\ell)$. Thus,
\begin{align}
y_1&:=\mathbb{E}[Y_1] = 2m(n/2-m)(a/n)  \sum_{\ell} \mu(\ell) w(\ell), \nonumber
\\
y_2&:=\mathbb{E}[Y_2] = 2m(n/2-m)(b/n) \sum_{\ell}  \nu(\ell) w(\ell). \nonumber
\end{align}
Define
\begin{align}
z_1 &:= 2m(n/2-m)(a/n)  \sum_{\ell } \mu( \ell ) w^2(\ell) , \nonumber \\
z_2&:= 2m(n/2-m)(b/n) \sum_{\ell }  \nu( \ell ) w^2(\ell). \nonumber
\end{align}
Then, for $0<\lambda \le \frac{1}{2} $,
\begin{align}
\mathbb{E} [\exp( -\lambda Y_1 )] &= \left[1 + \frac{a}{n} \sum_{\ell} (\eexp^{-\lambda w(\ell)} -1  )  \mu(\ell)  \right]^{2m(n/2-m)} \nonumber \\
& \le \exp \left[ 2m(n/2-m) \frac{a}{n}  \sum_{\ell} (\eexp^{-\lambda w(\ell)} -1  )  \mu(\ell)  \right] \nonumber \\
& \le \exp \left[  2m (n/2-m) \frac{a}{n}  \sum_\ell   \left(-\lambda w(\ell) + 2 \lambda^2 w^2(\ell) \right)  \mu(\ell)  \right] \nonumber \\
&= \exp( -\lambda y_1 + 2\lambda^2 z_1 ), \nonumber
\end{align}
where the first inequality follows from the fact that $1+x \le e^x$ and the second one follows from the fact that $e^x \le 1+x+ 2x^2$ for $|x|\le 1/2$.
The Chernoff bound gives that for $0< \lambda \le \frac{1}{2} $,
\begin{align}
\mathbb{P} ( Y_1 \le (1 -t)y_1 ) &\le \mathbb{E} [\exp( -\lambda Y_1 )] \exp ( (1-t) \lambda y_1) \nonumber \\
& \le \exp(-t \lambda y_1 + 2 \lambda^2 z_1 ). \label{eq:chernoff}
\end{align}

We define $\mathbb{E} [W_\mu] \triangleq \sum_\ell \mu(\ell )w(\ell )$ and $\mathbb{E}
[W_\mu^2] \triangleq \sum_\ell  \mu(\ell )w^2(\ell )$.
Let $t_1^2 = (64\ln 2)\frac{1+\epsilon}{1-\epsilon} \frac{1}{a} \frac{\mathbb{E}[W_\mu^2]}{\left( \mathbb{E}
[W_\mu] \right)^2}$ and $\lambda = \frac{t_1y_1}{4z_1}$. We first check that with
these values, we have $\lambda\leq 1/2$:
\begin{eqnarray*}
\lambda \leq \frac{1}{2} &\Leftrightarrow&
t_1\leq\frac{2\mathbb{E}[W_\mu^2]}{\mathbb{E}[W_\mu]}\\
&\Leftrightarrow& \frac{1+\epsilon}{1-\epsilon}\frac{8\ln 2}{a}\leq
\mathbb{E}[W^2_\mu].
\end{eqnarray*}
Thanks to the assumption made in Theorem \ref{ThmMinBisection}, we can
find $\epsilon$ sufficiently small such that this last inequlity is
valid. Notice that $\frac{t_1^2y_1^2}{8z_1}  \geq (1+\epsilon)^2 n \ln 2.$ It follows from \prettyref{eq:chernoff} that
\begin{align}
\mathbb{P} ( Y_1\le (1 -t_1)y_1 ) = \exp \left( - \frac{t_1^2 y_1^2}{8z_1} \right) \le 2^{-n (1+\epsilon)}. \nonumber
\end{align}
Since there are $\binom{n/2}{m}\binom{n/2}{m}  \le 2^{n}$ different $\sigma$ with $m(\sigma)=m$, a simple union bound yields that as $n \to \infty$,
\begin{align}
\mathbb{P} \left(\exists \sigma: (1-\epsilon)n/4 \le m(\sigma) \le n/4,  Y_{1} \le  (1 -t_1)y_1  \right) \to 0. \nonumber
\end{align}
Similarly, let $t_2^2= (64 \ln 2) \frac{1+\epsilon}{1-\epsilon} \frac{1}{b} \frac{\mathbb{E}[W_\nu^2]}{ \left( \mathbb{E}
[W_\nu]\right)^2}$ with $\mathbb{E} [W_\nu] \triangleq \sum_\ell \nu(\ell)w(\ell)$ and $\mathbb{E} [W_\nu^2] \triangleq \sum_\ell \nu(\ell)w^2(\ell)$. Then
\begin{align}
\mathbb{P} \left(\exists \sigma: (1-\epsilon)n/4 \le m(\sigma) \le n/4,  Y_{2} \ge  (1 +t_2)y_2  \right) \to 0. \nonumber
\end{align}
With $\epsilon$ sufficiently small, a.a.s.
\begin{align}
Y_1 -Y_2 & \ge (1-t_1)y_1 - (1+t_2)y_2 \nonumber \\
& = y_1-y_2 -\frac{2m}{n}(n/2-m)\sqrt{\frac{1+\epsilon}{1-\epsilon}(64\ln 2)} \left( \sqrt { a\mathbb{E}[W_\mu^2] } + \sqrt{ b\mathbb{E}[W_\nu^2] }\right) \nonumber \\
& \ge \frac{2m}{n}(n/2-m)\left( a\mathbb{E}[W_\mu]-b\mathbb{E}[W_\nu] -\sqrt{\frac{1+\epsilon}{1-\epsilon}(128\ln 2)} \sqrt { \left(  a\mathbb{E}[W_\mu^2]  +  b\mathbb{E}[W_\nu^2] \right)}\right) \nonumber
\end{align}
which is larger than zero as soon as $\epsilon$ is sufficiently small
and (\ref{EqMinBisectionCondition}) is satisfied.

By Cauchy-Schwartz inequality,
\begin{align}
\left(\sum_\ell (a\mu(\ell)-b\nu(\ell) )w (\ell) \right)^2 \le  2\tau \sum_\ell (a\mu(\ell)+b \nu(\ell)) w^2(\ell) \nonumber
\end{align}
with equality achieved when
$w(\ell)=\frac{a\mu(\ell)-b\nu(\ell)}{a\mu(\ell)+b\nu(\ell)}$. This completes the
proof.

\subsection{Proof of \prettyref{thm:SBMSDPCorrelated}} \label{SectionSDPpf}
Without loss of generality, assume \prettyref{eq:clustersizebalance} holds for $\sigma^\ast$.
Let $Y^\ast=\sigma^\ast (\sigma^\ast)^\top$.  By the optimality of $\widehat{Y}_{\SDP}$,
\begin{align}
0 \le \Iprod{W}{\widehat{Y}_{\SDP}} - \Iprod{W}{Y^\ast} = \Iprod{\mathbb{E}[W]}{\widehat{Y}_{\SDP}- Y^\ast} + \Iprod{W-\mathbb{E}[W]}{\widehat{Y}_{\SDP}- Y^\ast}. \label{eq:OptimalitySDP}
\end{align}
Since $\mathbb{E}[W]= \frac{\alpha}{n} \allones + \frac{\beta}{n} Y^\ast - \frac{\alpha+\beta}{n} \identity$ with $\alpha, \beta$ defined in~\eqref{EqDefAlphaBeta}, and $\widehat{Y}_{\SDP}$ is a feasible solution to \prettyref{eq:SBMMB2},
\begin{align*}
\Iprod{\mathbb{E}[W]}{\widehat{Y}_{\SDP}- Y^\ast} =\frac{\beta}{n} \Iprod{Y^\ast}{\widehat{Y}_{\SDP}- Y^\ast}- \frac{\alpha}{n} \Iprod{\allones}{Y^\ast} \le \frac{\beta}{n} \Iprod{Y^\ast}{\widehat{Y}_{\SDP}- Y^\ast},
\end{align*}
where the last inequality holds because $\Iprod{\allones}{Y^\ast} \ge 0$. 
In view of \prettyref{eq:OptimalitySDP}, it follows that
\begin{align}
\frac{\beta}{n} \Iprod{Y^\ast}{ Y^\ast - \widehat{Y}_{\SDP}} \le \Iprod{W-\mathbb{E}[W]}{\widehat{Y}_{\SDP}- Y^\ast}. \label{eq:correlationbound}
\end{align}
Notice that
\begin{align*}
\fnorm{ Y^\ast - \widehat{Y}_{\SDP} }^2 =\fnorm{Y^\ast}^2 + \fnorm{ \widehat{Y}_{\SDP}}^2 - 2\Iprod{Y^\ast}{\widehat{Y}_{\SDP}} \le 2 \left(n^2 - \Iprod{Y^\ast}{\widehat{Y}_{\SDP}} \right) = 2 \Iprod{Y^\ast}{ Y^\ast - \widehat{Y}_{\SDP}}.
\end{align*}
It follows from \prettyref{eq:correlationbound} that
\begin{align}
\frac{\beta}{2n} \fnorm{ Y^\ast - \widehat{Y}_{\SDP}} ^2 \le    \Iprod{W-\mathbb{E}[W]}{\widehat{Y}_{\SDP}- Y^\ast} \le  | \Iprod{W-\mathbb{E}[W]}{\widehat{Y}_{\SDP}}| + | \Iprod{W-\mathbb{E}[W]}{Y^\ast} | . \label{eq:fnormnbound}
\end{align}

To upper bound $| \Iprod{W-\mathbb{E}[W]}{Y^\ast} |$, Notice that
\begin{align*}
 \Iprod{W-\mathbb{E}[W]}{Y^\ast} =2 \sum_{i<j} Y^\ast_{ij} \left( W_{ij} - \mathbb{E}[W_{ij}] \right).
\end{align*}
Let $\sigma^2= \sum_{i<j} \var [ W_{ij} ] = (1+o(1)) \frac{n}{2} \sum_{\ell} w^2(\ell) ( a \mu(\ell) +b \nu(\ell) ) .$ By the Bernstein inequality given in \prettyref{thm:Bernstein}, for any $t>0$,
\begin{align*}
\mathbb{P} \left\{ \bigg| \sum_{i<j} Y^\ast_{ij} \left( W_{ij} - \mathbb{E}[W_{ij}] \right)  \bigg|  \ge \sqrt{ 2 \sigma^2 t} + \frac{2}{3} t \right\} \le 2 \eexp^{-t}.
\end{align*}
Letting $t=\log n = o (\sigma^2)$, it follows that with probability at least $1-2n^{-1}$,
\begin{align*}
\big| \sum_{i<j} Y^\ast_{ij} \left( W_{ij} - \mathbb{E}[W_{ij}] \right)  \big| \le (1+o(1)) \sqrt{n \log n\sum_{\ell} w^2(\ell) ( a \mu(\ell) +b \nu(\ell) )},
\end{align*}
and thus $| \Iprod{W-\mathbb{E}[W]}{Y^\ast} | \le (2+o(1))\sqrt{n \log n\sum_{\ell} w^2(\ell) ( a \mu(\ell) +b \nu(\ell) )}$ with probability at least $1-2n^{-1}$.

We bound $| \Iprod{W-\mathbb{E}[W]}{\widehat{Y}_{\SDP}} |$ next.
It follows from Grothendieck's inequality \cite[Theorem 3.4]{Vershynin14} that
\begin{align*}
| \Iprod{W-\mathbb{E}[W]}{\widehat{Y}_{\SDP}}|  \le \sup_{Y \succeq 0, \diag{Y}=\identity} |  \Iprod{W-\mathbb{E}[W]}{Y}| \le K_{\rm G} \|W-\mathbb{E}[W]\|_{\infty \to 1},
\end{align*}
where $K_{\rm G}$ is an absolute constant known as \emph{Grothendieck constant} and it is known that $K_{\rm G} < \frac{\pi}{2 \ln (1+ \sqrt{2}) } \le 1.783$.
Moreover,
\begin{align*}
\|W-\mathbb{E}[W]\|_{\infty \to 1} \triangleq \sup_{ x: \|x \|_{\infty} \le 1} \| (W-\mathbb{E}[W]) x \|_1 &= \sup_{ x, y \in \{ \pm 1\}^n } x^\top (W-\mathbb{E}[W]) y \\
&= \sup_{x, y \in \{\pm 1\}^n}  \sum_{i<j} \left( W_{ij} - \mathbb{E}[W_{ij}] \right) (x_iy_j + x_j y_i).
\end{align*}
For any fixed $x, y \in \{\pm 1\}^n$, using the Bernstein inequality, we have for any $t>0$,
\begin{align*}
\mathbb{P} \left\{ \sum_{i<j} \left( W_{ij} - \mathbb{E}[W_{ij}] \right) (x_iy_j + x_j y_i) \ge \sqrt{ 8 \sigma^2 t} + \frac{4}{3} t \right\} \le \eexp^{-t}.
\end{align*}
%Letting $t=n(1+\epsilon) 2 \ln 2 $ for some arbitrarily small constant $\epsilon>0$,
Hence, for arbitrarily small constant $\epsilon>0$, with probability at least $2^{-2(1+\epsilon)n}$,
\begin{align*}
\sum_{i<j} \left( W_{ij} - \mathbb{E}[W_{ij}] \right) (x_iy_j + x_j y_i) &\le n \left(  \sqrt{ 8 \ln 2 (1+\epsilon)  \sum_{\ell} w^2(\ell) ( a \mu(\ell) +b \nu(\ell) ) }  + \frac{8\ln 2 (1+\epsilon) }{3} \right)   \\
& \overset{(a)}{\le}  \frac{4n}{3}   \sqrt{ 8 \ln 2 (1+\epsilon)  \sum_{\ell} w^2(\ell) ( a \mu(\ell) +b \nu(\ell) ) } ,
\end{align*}
where $(a)$ follows from the technical assumption $\sum_{\ell} w^2(\ell) ( a \mu(\ell) +b \nu(\ell) ) > 8 \ln 2$.
It follows from the union bound that with probability at least $1-4^{-\epsilon n}$,
\begin{align*}
\|W-\mathbb{E}[W]\|_{\infty \to 1} \le \frac{4n}{3}   \sqrt{ 8 \ln 2 (1+\epsilon)  \sum_{\ell} w^2(\ell) ( a \mu(\ell) +b \nu(\ell) ) }.
\end{align*}
In view of \prettyref{eq:fnormnbound}, with probability at least $1-4^{-\epsilon n}-2n^{-1}$,
\begin{align}
 \frac{1}{n^2} \fnorm{ Y^\ast - \widehat{Y}_{\SDP}} ^2 &\le (1+o(1)) \frac{8K_{\rm G} }{3\beta} \sqrt{ 8 \ln 2 (1+\epsilon)  \sum_{\ell} w^2(\ell) ( a \mu(\ell) +b \nu(\ell) ) }   \nonumber \\
  & \overset{(a)}{\le} (1+o(1)) 32 \sqrt{\ln 2 (1+ \epsilon) }  \frac{ \sqrt{ \sum_{\ell} w^2(\ell) ( a \mu(\ell) +b \nu(\ell) ) } }{ \sum_{\ell} w(\ell) (a \mu (\ell) - b \nu( \ell) } \nonumber \\ 
  & \overset{(b)}{\le} (1-\epsilon) \frac{1}{16}, \label{eq:fnormupperbound}
\end{align}
where $(a)$ follows by $ \sqrt{2} K_G \le 3 $ and the definition of $\beta$ given in~\eqref{EqDefAlphaBeta}; $(b)$ holds by invoking \prettyref{eq:SBMSDPCondition} and letting $\epsilon$ be sufficiently small.

Recall that $y$ is an eigenvector of $\widehat{Y}_{\SDP}$ corresponding to the largest eigenvalue and $\|y\| =\sqrt{n}$.
By Davis-Kahan sin$\theta$ theorem stated in  \prettyref{lmm:daviskahan},
\begin{align*}
\frac{1}{\sqrt{n} }\min \{ \| \sigma^\ast - y \|, \| \sigma^\ast +y \| \} \le \frac{2 \sqrt{2} \|\widehat{Y}_{\SDP}- Y^\ast \| }{n} \le \frac{2 \sqrt{2} \fnorm{\widehat{Y}_{\SDP}- Y^\ast}}{n } .
\end{align*}
Note that for any $x \in \mathbb{R}^n$, Hamming distance $d (\sigma^\ast, \sign( x ) ) \le \| \sigma^\ast -x \|^2.$
It follows that
\begin{align*}
\frac{1}{n} \min \{ d (\sigma^\ast, \sign( y ) ), d (\sigma^\ast, \sign(-y) ) \} \le \frac{8 \fnorm{\widehat{Y}_{\SDP}- Y^\ast}}{n^2},
\end{align*}
and the theorem holds in view of \prettyref{eq:fnormupperbound}.

\subsection{Proof of Theorem \ref{ThmSpectralSBM}} \label{SectionSpectralpf}

Recall that $W'$ is the weighted adjacency matrix after removal of edges incident to vertices with high degrees and $D'=W'-\frac{\alpha}{n} \allones$.
Define $\hat{D^\prime}$ as the best rank-1 approximation of $D^\prime$ such that $\hat{D^\prime} = v_1 x x^\top$ with
$\|x\|=1$.  Recall that $\mathbb{E}[D|\sigma] = \frac{\beta}{n} \sigma \sigma^\top$. Applying Davis-Kahan $\sin \theta$ theorem restated in \prettyref{lmm:daviskahan} with $D^\prime$ and $\mathbb{E}[D|\sigma]$ gives:
\begin{eqnarray*}
\min \{
\|\frac{\sigma}{\sqrt{n}}-x\|,\|\frac{\sigma}{\sqrt{n}}+x\|
\}\leq \frac{2\sqrt{2}}{|\beta| } \|D'-\mathbb{E}[D|\sigma]\| .
\end{eqnarray*}
Since Hamming distance $d (\sigma, \sign (x ) \le  \| \sigma - \sqrt{n} x \|^2$, it follows that
\begin{align}
\frac{1}{n} \min \{ d (\sigma, \sign (x ), d(\sigma, -\sign( x ) ) \} \le \frac{8}{\beta^2}\|D^\prime-\mathbb{E}[D|\sigma]\|^2 = \frac{8}{\beta^2}\|W'- \mathbb{E}[W | \sigma]\|^2.   \label{eq:overlapbound}
\end{align}
%Notice that
%\begin{align}
%\| \hat{D^\prime} - \mathbb{E}[D | \sigma] \| \le \| \hat{D^\prime} - D'\| + \| D'- \mathbb{E}[D | \sigma] \| \le 2 \| \hat{D^\prime} - D' \|
%= 2  \|W^\prime- \mathbb{E}[W | \sigma]\|.  \nonumber
%\end{align}
%Lemma with $\hat{D}$ and $\mathbb{E}[D|\sigma]$ gives:
%\begin{eqnarray*}
%\min \{
%\|\frac{\sigma}{\sqrt{n}}-\hat{x}\|,\|\frac{\sigma}{\sqrt{n}}+\hat{x}\|
%\}\leq \frac{\sqrt{2}}{\beta} \|\hat{D}-\mathbb{E}[D|\sigma]\|.
%\end{eqnarray*}
%Hence thanks to Lemma \ref{lem:distsign}, we have
%\begin{eqnarray}
%\frac{1}{n} \min \{ d (\sigma, {\rm sign}\hat{x} ), d(\sigma, -{\rm
%  sign}\hat{x}) \} \leq \frac{2}{\beta^2}\|\hat{D}-\mathbb{E}[D|\sigma]\|^2\label{eq:dist}.
%\end{eqnarray}
%Also, note that
%\begin{eqnarray}
%\| \hat{D'}- \mathbb{E}[D | \sigma] \|
%& \le&  \|D'- \mathbb{E}[D | \sigma] \| +  \|D'- \hat{D'} \| \nonumber \\
%& \overset{(a)}{\le} & 2\|D'- \mathbb{E}[D | \sigma] \|  = 2 \|W'- \mathbb{E}[W | \sigma]\|, \label{EqSpectralNormFrobeniusNorm}
%\end{eqnarray}
%where $(a)$ holds because $\hat{D'}$ is the best rank-$1$
%matrix approximation to $D'$ in spectral norm.
\prettyref{lmm:spectrumsparse} implies that  a.a.s.\ $\|W'- \mathbb{E}[W | \sigma]\| \le C\sqrt{a+b}$ for some universal positive constant $C$.
Hence, in view of \prettyref{eq:overlapbound}, we get
\begin{eqnarray*}
\frac{1}{n} \min \{ d (\sigma, \sign ( \hat{x} ), d(\sigma, - \sign (\hat{x}) )  \} \leq 8C^2 \frac{a+b}{\beta^2},
\end{eqnarray*}
and the theorem follows.

%Hence without triming and using Vu's results, we have
%\begin{eqnarray*}
%\frac{1}{n}\min \{ d (\sigma, {\rm sign}\hat{x} ), d(\sigma, -{\rm
%  sign}\hat{x}) \} &\leq& \frac{32}{\beta^2}\sum_\ell
%w(\ell)^2\left(a\mu(\ell)+b\nu(\ell)\right)\\
%&=& 128 \frac{\sum_\ell w(\ell)^2\left(a\mu(\ell)+b\nu(\ell)\right)}{\left(\sum_\ell
%w(\ell)\left(a\mu(\ell)-b\nu(\ell)\right)\right)^2}.
%\end{eqnarray*}

\subsection{Proof of Theorem \ref{ThmSpectralSBM-large}}

%The proof of Theorem \ref{ThmSpectralRegular} starts here.
%\begin{lemma}\label{lem:distsign}
%For any $\sigma\in \{-1,+1\}^n$ and $\beta\in \mathbb{R}^n$ with
%$\|\beta\|=1$, we have:
%\begin{align}
%d(\sigma, {\rm sign}(\beta)) \le n \| \frac{1}{\sqrt{n}} \sigma - \beta \|^2, \label{EqBoundHammingDistance}
%\end{align}
%where $d$ is the Hamming distance.
%\end{lemma}
%\begin{proof}
%We have
%\begin{eqnarray*}
%d(\sigma, {\rm sign}(\beta)) &\leq& \sum_{i=1}^n
%\ind\left(\sigma_i\beta_i\leq 0\right).
%\end{eqnarray*}
%Let $S= \{i\in [n], \sigma_i\beta_i\leq 0\}$, so that we have
%\begin{eqnarray*}
%\sum_{i=1}^n \left( \frac{\sigma_i}{\sqrt{n}} - \beta_i\right)^2 &=&
%2\left( 1-\frac{1}{\sqrt{n}}\sum_{i=1}^n \sigma_i \beta_i\right)\\
%&\geq& 2\left( 1+\frac{1}{\sqrt{n}}\sum_{i\notin S} |\beta_i|\right).
%\end{eqnarray*}
%By Cauchy-Schwarz inequality, we have
%\begin{eqnarray*}
%\frac{1}{n}\left( \sum_{i\notin S} |\beta_i|\right)^2\leq
%\frac{n-|S|}{n}\sum_{i=1}^n\beta_i^2\leq \left( 1-\frac{|S|}{2n}\right)^2.
%\end{eqnarray*}
%Hence, we proved that
%\begin{eqnarray*}
%\sum_{i=1}^n \left( \frac{\sigma_i}{\sqrt{n}} - \beta_i\right)^2 \geq \frac{|S|}{n},
%\end{eqnarray*}
%which is exactly the claim of the lemma.
%\end{proof}
The proof follows the same steps as for Theorem \ref{ThmSpectralSBM},
except that we are able to strengthen Lemma \ref{lmm:spectrumsparse} thanks to a
result of Vu \cite{vu05}. Note that the variance of the elements of
$W$ is upper bounded by $\frac{1}{n} \sum_\ell w^2(\ell) \left(
  a\mu(\ell)+b\nu(\ell)\right)$ so that by Theorem 1.4 in \cite{vu05},
we get
\begin{lemma}
Under the conditions of Theorem \ref{ThmSpectralSBM-large}, we have
\begin{eqnarray*}
\|W- \mathbb{E}[W | \sigma] \| \leq 2 \sqrt{\sum_\ell w^2(\ell)\left(
  a\mu(\ell)+b\nu(\ell)\right)}\quad a.a.s.
\end{eqnarray*}
\end{lemma}

\subsection{Proof of Theorem \ref{ThmNonReconstruction}} \label{SectionNonReconstruction}
Consider a Galton-Watson tree $T$ with Poisson offspring distribution with mean $\frac{a+b}{2}$. The type of the root $\rho$ is chosen from $\{ \pm 1\}$ uniformly at random. Each child has the same type as its parent with probability $\frac{a}{a+b}$ and a different type with probability $\frac{b}{a+b}$ . Every edge $(u,v)$ is  labeled  at random with distribution $\mu$ if $\sigma_u=\sigma_v$ and $\nu$ otherwise. Let $T_R$ denote the Galton-Watson tree $T$ up to depth $R$ and $\partial T_R$ denote the set of leaves of $T_R$. Let $G_R$ denote the subgraph of $G$ induced  by vertices up to distance $R$ from $\rho$ and $\partial G_R$ be the set of vertices at distance $R$ from $\rho$.

The following lemma similar to Proposition 4.2 in \cite{Mossel12} establishes a coupling between the local neighborhood of $\rho$ and the labeled Galton-Watson tree rooted at $\rho$.

\begin{lemma}\label{PropCouplingTree}
Let $R=R(n)=\lfloor \frac{\log n}{10 \log (2(a+b) )} \rfloor $, then there exists a coupling such that a.a.s.
\begin{align}
(G_R,L_{G_R},\sigma_{G_R})=(T_R, L_{T_R},\sigma_{T_R}), \nonumber
\end{align}
where $L_{G_R}$ and $\sigma_{G_R}$ denote the labels and types on the subgraph $G_R$, respectively.
\end{lemma}
\begin{proof}
See proof in Section \ref{PfPropCouplingTree}.
\end{proof}

To ease notation, we omit the shorthand a.a.s. in the sequel. To prove Theorem \ref{ThmNonReconstruction}, it suffices to show that $\text{Var}(\sigma_{\rho} |G,L, \sigma_v) \to 1$.
%By the monotonicity property of conditional variance, it further reduces to show that $\text{Var} (\sigma_{\rho} |G, L, \sigma_v, \sigma_{\partial G_R} )  \to 1$. 
By the law of total variance, 
\begin{align*}
\text{Var}(\sigma_{\rho} |G,L, \sigma_v)  = \mathbb{E}_{\sigma_{\partial G_R }} \left[  \text{Var} ( \sigma_{\rho} | G, L ,\sigma_v, \sigma_{\partial G_R} )  \right] + \text{Var}_{\sigma_{\partial G_R }}   \left [  \mathbb{E} \left [ \sigma_ \rho |  G, L, \sigma_v, \sigma_{\partial G_R} \right] \right].
\end{align*}
Hence, it further reduces to show that $\text{Var} (\sigma_{\rho} |G, L, \sigma_v, \sigma_{\partial G_R} )  \to 1$. 

Let $R$ be as in Lemma \ref{PropCouplingTree}, then $G_R=o(\sqrt{n})$ and thus $v \notin G_R$. Lemma 4.7 in \cite{Mossel12} shows that $\sigma_{\rho}$ is asymptotically independent with $\sigma_v$ conditionally on $\sigma_{\partial G_R}$. Hence,
\begin{align}
\text{Var}(\sigma_{\rho}|G,L, \sigma_v, \sigma_{\partial G_R}) \to \text{Var} (\sigma_{\rho}|G,L, \sigma_{\partial G_R}). \nonumber
\end{align}
Let $G_R^c$ denote the subgraph of $G$ induced by edges not in $G_R$, and $L_{G_R^c}$ denote the set of labels on $G_R^c$. 
Recall that $V(G_{R-1})$ and $V(G_R^c)$ denote the set of vertices in $G_{R-1}$ and $G_R^c$, respectively. 
Let $S\triangleq V(G_{R-1}) \setminus \{ \rho \}$ and 
$T \triangleq  V(G_R^c) \setminus \partial G_R.$ Then $\{\rho\} \cup \partial G_R \cup S \cup T = V(G)$. 
Notice that conditional on $( G_R, L_{G_R}, \sigma_{\partial G_R} )$, $\sigma_\rho$ is independent of $(G_R^c, L_{G_R^c} )$. 
In particular, 
\begin{align*}
& \prob{ \sigma_ \rho  | G_R, L_{G_R}, \sigma_{\partial G_R} } \\
& = \frac{ \sum_{G_R^c, L_{G_R^c} } \prob{\sigma_\rho, G, L, \sigma_{\partial G_R} } }  {  \sum_{G_R^c, L_{G_R^c} } \prob{G, L, \sigma_{\partial G_R} }  } \\
& =\frac{ \sum_{G_R^c, L_{G_R^c} } \left(  \sum_{\sigma_S } \prod_{u,v \in V(G_R): u<v } \phi_{uv}   \right) \left(  \sum_{\sigma_T }  
\prod_{u,v \in T: u<v} \phi_{uv} \prod_{u \in \partial G_R, v \in T } \phi_{uv} \right) }
 {   \sum_{G_R^c, L_{G_R^c} } \left(  \sum_{\sigma_\rho} \sum_{\sigma_S } \prod_{u,v \in V(G_R): u<v } \phi_{uv}   \right) \left(  \sum_{\sigma_T }  
\prod_{u,v \in T: u<v} \phi_{uv} \prod_{u \in \partial G_R, v \in T } \phi_{uv} \right) } \\
 & \overset{(a)} { =} \frac{   \left(  \sum_{\sigma_S } \prod_{u,v \in V(G_R): u<v } \phi_{uv}   \right) \sum_{G_R^c, L_{G_R^c} }  \left(  \sum_{\sigma_T }  
\prod_{u,v \in T: u<v} \phi_{uv} \prod_{u \in \partial G_R, v \in T } \phi_{uv} \right)
   } {   \left(  \sum_{\sigma_\rho} \sum_{\sigma_S } \prod_{u,v \in V(G_R): u<v } \phi_{uv}   \right)   \sum_{G_R^c, L_{G_R^c} }  \left(  \sum_{\sigma_T }  
\prod_{u,v \in T: u<v} \phi_{uv} \prod_{u \in \partial G_R, v \in T } \phi_{uv} \right)  } \\
 & =\frac{    \sum_{\sigma_S } \prod_{(u,v) \in V(G_R): u<v } \phi_{uv} } { 
  \sum_{\sigma_\rho}  \sum_{\sigma_S } \prod_{(u,v) \in G_R: u<v } \phi_{uv}    }  \\
   & = \frac{ \left( \sum_{\sigma_S  } \prod_{(u,v) \in V(G_R): u<v } \phi_{uv}  \right) \left(  \sum_{\sigma_T }  
\prod_{u,v \in T: u<v} \phi_{uv} \prod_{u \in \partial G_R, v \in T } \phi_{uv} \right)  } {  \left( \sum_{\sigma_\rho} \sum_{\sigma_S } \prod_{(u,v) \in V(G_R): u<v } \phi_{uv} \right)   
 \left(  \sum_{\sigma_T }  
\prod_{u,v \in T: u<v} \phi_{uv} \prod_{u \in \partial G_R, v \in T } \phi_{uv} \right)   }  \\
 & = \frac{  \prob{\sigma_\rho, G, L, \sigma_{\partial G_R} }  }{ \prob{G, L, \sigma_{\partial G_R} }  } =\prob{ \sigma_ \rho  | G, L, \sigma_{\partial G_R} } ,
\end{align*}
where $(a)$ holds because $\sum_{\sigma_S } \prod_{(u,v) \in V(G_R): u<v } \phi_{uv} $ does not depend on $G_R^c$ and $L_{G_R^c}$. It follows that
\begin{align}
\text{Var} (\sigma_{\rho}|G,L, \sigma_{\partial G_R})= \text{Var} (\sigma_{\rho}|G_R,L_{G_R}, \sigma_{\partial G_R}). \nonumber
\end{align}
Lemma \ref{PropCouplingTree} implies that
\begin{align}
\text{Var} (\sigma_{\rho}|G_R,L_{G_R}, \sigma_{\partial G_R}) \to \text{Var} (\sigma_{\rho}|T_R,L_{T_R}, \sigma_{\partial T_R}). \nonumber
\end{align}
For the labeled Galton-Watson tree, it was shown in \cite{Heimlicher12} that if $\tau<1$, the types of the leaves provide no information about the type of the root when the depth $R \to \infty$, i.e.,
\begin{align}
\mathbb{P} ( \sigma_{\rho} =+1 |T, L, \sigma_{\partial T_R} ) \to \frac{1}{2}. \nonumber
\end{align}
Hence, $\text{Var} (\sigma_{\rho}|T_R,L_{T_R}, \sigma_{\partial T_R}) \to 1$
and the theorem follows.

\subsection{Proof of Theorem \ref{ThmACER}} \label{SectionHypTesting}
We introduce some necessary notations. For a graph $G$ with $n$ vertices and labeled edges, denote a $k$-sequence of labels by $ [\ell]_k=(\ell_1,\ell_2,\ldots, \ell_k) \in \mathcal{L}^k$. A cycle in $G$ is called a $k$-cycle with labels $[\ell]_k$, if starting from the vertex with the minimum index and ending at its neighbor with the smaller index among its two neighbors, the sequence of labels on edges is given by $[\ell]_k$. Let $X_n([\ell]_k)$ denote the number of $k$-cycles with labels $[\ell]_k$ in $G$. Let $(X)_j=X(X-1)\cdots (X-j+1)$ for integers $X$ and $1 \le j \le X$. Then $(X_n([\ell]_k))_j$ is the number of ordered $j$-tuples of $k$-cycles with labels $[\ell]_k$ in $G$. The product $\prod_{[\ell]_k}$ is assumed to taken over all possible sequences of labels with length $k$. The following lemma gives the asymptotic distribution of the number of $k$-cycles with labels $[\ell]_k$.

\begin{lemma} \label{LemmaNumCycles}
For any fixed integer $m \ge 3$, $\{X_n([\ell]_k): [\ell]_k \in \calL^k \}_{k=3}^{m}$ jointly converge to independent Poisson random variables with mean $\lambda([\ell]_k)$ under graph distribution $\mathbb{P}_n^\prime$, and $\xi([\ell]_k)$ under graph distribution $\mathbb{P}_n$, where
\begin{align}
&\lambda([\ell]_k)= \frac{1}{2^{k+1} k } \prod_{i=1}^k (a\mu(\ell_i)+b\nu(\ell_i)), \nonumber \\
&\xi([\ell]_k) = \frac{1}{2^{k+1} k } \left( \prod_{i=1}^k (a\mu(\ell_i)+ b\nu(\ell_i)) + \prod_{i=1}^k (a\mu(\ell_i)-b\nu(\ell_i)) \right). \nonumber
\end{align}
%Moreover, for $k=O(\log^{1/4} n)$ and each fixed $[\ell]_k$, $X_n[\ell]_k$ converges to Poisson random
\end{lemma}
%\begin{proof}
%See proof in Section \ref{PfLemmaNumCycles}.
%\end{proof}

We are ready to prove  Theorem \ref{ThmACER}.  The first part of Theorem \ref{ThmACER} is proved using Lemma \ref{LemmaNumCycles} and Chebyshev inequality. Define $\eta([\ell]_k)=\xi([\ell]_k) / \lambda([\ell]_k) -1$ and $X_k= \sum_{[\ell]_k} X([\ell]_k) \eta([\ell]_k)$. Then, by Lemma  \ref{LemmaNumCycles}, as $n \to \infty$,
\begin{align}
\mathbb{E}_{\mathbb{P}^\prime}[X_k] &=\sum_{[\ell]_k}  \lambda([\ell]_k) \eta([\ell]_k),   \nonumber\\
\mathbb{E}_{\mathbb{P}}[X_k] &=  \sum_{[\ell]_k} \lambda([\ell]_k ) \eta([\ell]_k) (1+ \eta([\ell]_k)). \nonumber
\end{align}
Note that
\begin{align}
2k \sum_{[\ell]_k} \lambda([\ell]_k) \eta^2([\ell]_k) =  \sum_{[\ell]_k} \prod_{s=1}^k \frac{(a \mu(\ell_s) - b \nu(\ell_s) )^2 }{ 2(a\mu(\ell_s) + b \nu(\ell_s)) } =  \left( \sum_{\ell \in \mathcal{L}} \frac{(a \mu(\ell) - b \nu(\ell) )^2 }{ 2(a\mu(\ell) + b \nu(\ell)) } \right)^k = \tau^k. \label{EqTau}
\end{align}
Therefore,
\begin{align}
\mathbb{E}_{\mathbb{P}}[X_k]- \mathbb{E}_{\mathbb{P}^\prime}[X_k]=  \sum_{[\ell]_k} \lambda([\ell]_k) \eta^2([\ell]_k)  = \tau^k/ (2k) ,  \nonumber
\end{align}
and
\begin{align}
\text{Var}_{\mathbb{P}^\prime}[X_k] &= \sum_{[\ell]_k} \lambda([\ell]_k) \eta^2([\ell]_k) =  \tau^k/ (2k),  \nonumber\\
\text{Var}_{\mathbb{P}}[X_k] &= \sum_{[\ell]_k} \xi([\ell]_k) \eta^2([\ell]_k)  \le  \tau^k/ k. \nonumber
\end{align}
Choose $\rho= \tau^k/(6k)$. By Chebyshev's inequality,
\begin{align*}
\mathbb{P}' \{X_k >  \mathbb{E}_{\mathbb{P}^\prime}[X_k] + \rho \} \le \frac{\text{Var}_{\mathbb{P}^\prime}[X_k] }{\rho^2} = \frac{18 k}{\tau^k}.
\end{align*}
Let $k$ increases with $n$ sufficiently slowly. Then since $\tau>1$, $X_k \le \mathbb{E}_{\mathbb{P}^\prime}[X_k] + \rho$ $\mathbb{P}^\prime$-a.a.s..
Similarly, $X_k \ge \mathbb{E}_{\mathbb{P}}[X_k]-\rho$  $\mathbb{P}$-a.a.s..
By definition of $\rho$, $\mathbb{E}_{\mathbb{P}}[X_k]-\rho > \mathbb{E}_{\mathbb{P}^\prime}[X_k] + \rho$. Set $A_n= \{X_k \le \mathbb{E}_{\mathbb{P}^\prime}[X_k] + \rho \}$, then $\mathbb{P}^\prime(A_n) \to 1$ and $\mathbb{P}(A_n) \to 0$.

The second part of Theorem \ref{ThmACER} is proved using the following small subgraph conditioning theorem, which is adapted from \cite[Theorem 9.12]{Janson11}.
\begin{theorem} \label{ThmSubgraphCond}
Let $Y_n= \frac{\mathbb{P}_n}{\mathbb{P}_n^\prime}$. If $\mathbb{P}_n$ and  $\mathbb{P}_n^\prime$ are absolutely contiguous for any fixed $n$, and
\begin{enumerate}
\item For each fixed $m \ge 3$, $\{X_n([\ell ]_k) \}_{k=3}^{m}$ converge jointly to independent Poisson variables with means $\lambda([\ell]_k)>0$ under distribution $\mathbb{P}_n^\prime$, and $\xi([\ell ]_k)$ under distribution $\mathbb{P}_n$;
\item $\sum_{k \ge 3} \sum_{[\ell ]_k} \lambda([\ell ]_k) \eta([\ell ]_k)^2 < \infty$;
\item $\mathbb{E}_{\mathbb{P}^\prime_n}[Y_n^2] \to \exp ( \sum_{k \ge 3} \sum_{[\ell ]_k} \lambda([\ell ]_k) \eta^2([\ell ]_k) )$ as $n \to \infty$,
\end{enumerate}
Then, $\mathbb{P}_n$ and $\mathbb{P}_n^\prime$ are contiguous.
\end{theorem}

In this paper, $\mathbb{P}_n$ and $\mathbb{P}_n^\prime$ are discrete distributions on the space of labeled graphs, and for any fixed $n$, 
$\mathbb{P}_n$ and $\mathbb{P}_n^\prime$ are absolutely continuous. 
Condition 1) is verified by Lemma \ref{LemmaNumCycles}.  Condition 2) holds because  in view of (\ref{EqTau}), 
\begin{align}
\sum_{k \ge 3} \sum_{[\ell] _k} \lambda([\ell ]_k) \eta^2([\ell ]_k) =  \sum_{k \ge 3} \frac{\tau^k}{2k} = 
-\frac{\log(1-\tau)+\tau+\tau^2/2 }{2}< \infty. \label{eq:expressiontau}
\end{align}
We are left to verify condition 3). By definition,
\begin{align}
Y_n(G,L)=2^{-n} \sum_{\sigma \in \{\pm 1\}^n} \prod_{(u,v): u<v} W_{u,v}(G,L,\sigma), \nonumber
\end{align}
where
\begin{align}
W_{uv}(G,L,\sigma)=\left \{
\begin{array}{rl}
 \frac{2 a \mu(l)}{a \mu(\ell) + b \nu(\ell) } & \text{if } \sigma_u=\sigma_v, (u,v) \in E(G), L_{uv}=\ell, \\
 \frac{2 b \nu(l)}{a \mu(\ell) +  b \nu(\ell ) } & \text{if } \sigma_u \neq \sigma_v, (u,v) \in E(G), L_{uv}=\ell, \\
 \frac{1-a/n}{1-(a+b)/(2n) } & \text{if } \sigma_u= \sigma_v, (u,v) \notin E(G), \\
 \frac{1-b/n}{1-(a+b)/(2n) } & \text{if } \sigma_u \neq \sigma_v, (u,v) \notin E(G),
\end{array} \right. \nonumber
\end{align}
Then, 
\begin{align}
Y_n^2=2^{-2 n} \sum_{\sigma,\delta \in \{\pm 1 \}^n }  \prod_{(u,v): u<v} W_{u,v}(G,L,\sigma) W_{u,v}(G,L,\delta).  \label{eq:Ysecondmoment}
\end{align}

\begin{lemma} \label{lmm:WV}
For any fixed $\sigma, \delta \in  \{\pm 1 \}^n $, if $\sigma_u \sigma_v =\delta_u \delta_v$, then
\begin{align}
\mathbb{E}_{\mathbb{P}^\prime_n}[W_{u,v}(G, L, \sigma) W_{u,v} (G, L, \delta) ] = 1+ \tau/n+ (a-b)^2/(4n^2)+ O(n^{-3}). \nonumber
\end{align}
Otherwise,
\begin{align}
\mathbb{E}_{\mathbb{P}^\prime_n}[W_{u,v}(G, L, \sigma) W_{u,v} (G, L, \delta)] = 1- \tau/n - (a-b)^2/(4n^2)+ O(n^{-3}). \nonumber
\end{align}
\end{lemma}
\begin{proof}
Suppose $\sigma_u \sigma_v =\delta_u \delta_v=1$. Then,
\begin{align}
& \mathbb{E}_{\mathbb{P}^\prime_n}[W_{u,v}(G, L, \sigma) W_{u,v} (G, L, \delta)] \nonumber \\
& = \sum_\ell \left( \frac{2 a \mu(\ell) }{a \mu(\ell) + b\nu(\ell ) } \right)^2 \frac{a\mu(\ell )+b\nu(\ell ) }{2n} + \left( \frac{1-a/n}{1-(a+b)/(2n) } \right)^2 \left( 1-\frac{a+b}{2n} \right) \nonumber \\
& = \frac{1}{n }\sum_\ell \frac{2a^2 \mu^2(\ell) }{a\mu(\ell)+b\nu(\ell)} +\left(1-\frac{a}{n} \right)^2 \left(1+ \frac{a+b}{2n} + \frac{(a+b)^2}{4n^2}+ O(n^{-3}) \right) \nonumber \\
& = 1+ \frac{1}{n} \sum_\ell  \left( \frac{2a^2 \mu^2(\ell) }{a\mu(\ell)+b\nu(\ell)} + \frac{b \nu(\ell) - 3a \mu(\ell) }{2} \right) + \frac{(a-b)^2}{4n^2} + O(n^{-3}) \nonumber \\
&= 1+ \frac{1}{n} \sum_\ell \frac{(a\mu(\ell )-b\nu(\ell ))^2 }{2(a\mu(\ell)+b\nu(\ell))} + \frac{(a-b)^2}{4n^2} + O(n^{-3}) \nonumber \\
&= 1+ \tau/n + (a-b)^2/(4n^2)+ O(n^{-3}). \label{eq:secondmoment}
\end{align}
%The computation for $\sigma_u \sigma_v =\delta_u \delta_v=-1$ is similar.
By symmetry, \prettyref{eq:secondmoment} holds for $\sigma_u \sigma_v =\delta_u \delta_v=-1$. 
Suppose $\sigma_u =\sigma_v$ and $\delta_u \neq \delta_v$. Then,
\begin{align}
& \mathbb{E}_{\mathbb{P}^\prime_n}[W_{u,v}(G, L, \sigma) W_{u,v} (G, L, \delta)] \nonumber \\
& = \sum_\ell \frac{4 ab \mu(\ell) \nu(\ell) }{ (a \mu(\ell) + b\nu(\ell))^2 } \frac{a\mu(\ell)+b\nu(\ell) }{2n} + \frac{(1-a/n)(1-b/n) }{(1-(a+b)/(2n) )^2} \left( 1-\frac{a+b}{2n} \right) \nonumber \\
&= 1- \frac{1}{n} \sum_\ell \frac{(a\mu(\ell )-b\nu(\ell))^2 }{2(a\mu(\ell)+b\nu(\ell))} - \frac{(a-b)^2}{4n^2} + O(n^{-3}) \nonumber \\
&= 1- \tau/n - (a-b)^2/(4n^2)+ O(n^{-3}). \nonumber
\end{align}
\end{proof}

In view of \prettyref{lmm:WV}, letting $S(\sigma, \delta) \triangleq \{ (u,v): u<v, \sigma_u \sigma_v =\delta_u \delta_v \}$
and $T(\sigma, \delta) \triangleq  \{ (u,v): u<v, \sigma_u \sigma_v \neq \delta_u \delta_v \}$,  and $\gamma_n \triangleq  \tau/n + (a-b)^2/(4n^2)+ O(n^{-3})$,  
it follows from \prettyref{eq:Ysecondmoment} that
\begin{align}
\expects{Y_n^2}{\mathbb{P}^\prime_n} =2^{-2 n} \sum_{\sigma,\delta \in \{\pm 1 \}^n }  \left( 1+ \gamma_n \right)^{|S(\sigma, \delta)| } 
\left(1-\gamma_n  \right)^{| T (\sigma, \delta) | }.  \label{eq:secondmoment2}
\end{align}
Define $\rho(\sigma, \delta)= \Iprod{\sigma}{\delta}$ and then $|S(\sigma, \delta)| = (n^2+ \rho^2)/4 - n/2$ and $| T(\sigma, \delta) | = (n^2- \rho^2) /4$. 
It follows from \prettyref{eq:secondmoment2} that
\begin{align}
\expects{Y_n^2}{\mathbb{P}^\prime_n}= \left( 1+ \gamma_n \right)^{n^2/4-n/2} \left( 1- \gamma_n \right)^{n^2/4}  2^{-2 n} \sum_{\sigma,\delta \in \{\pm 1 \}^n } \left( 1+ \gamma_n \right)^{\rho^2/4} \left( 1- \gamma_n \right)^{-\rho^2/4}.  \label{eq:secondomement3}
\end{align}
Taylor expansion yields 
\begin{align}
\left( 1+ \gamma_n \right)^{n^2/4-n/2} \left( 1- \gamma_n \right)^{n^2/4}  &= \left(1 + O(n^{-1} ) \right)\exp \left[ -\tau^2/4 -\tau/2 \right],  \nonumber \\
 \left( 1+ \gamma_n \right)^{\rho^2/4} \left( 1- \gamma_n \right)^{-\rho^2/4} &=\exp \left[ \frac{ \rho^2}{n} ( \tau/2 + O(n^{-1} ) )  \right].  \label{eq:Taylor}
\end{align}
Combing \prettyref{eq:secondomement3} and \prettyref{eq:Taylor}, we get that
\begin{align}
\expects{Y_n^2}{\mathbb{P}^\prime_n} = \left(1 + O(n^{-1} ) \right) \exp \left[ -\tau^2/4 -\tau/2 \right] \expect{ \eexp^{ Z_n^2 ( \tau /2 + O(n^{-1} ))  }  },  \label{eq:secondmement4}
\end{align}
where $Z_n = \frac{1}{\sqrt{n}} \Iprod{\sigma}{\delta}$ and $\sigma, \delta$ are independently and uniformly distributed over $\{\pm 1\}^n$. 
Let $Z$ denote a standard Gaussian random variable. Then central limit theorem implies that $Z_n$ converges to $Z$ in distribution. 
Since $ x \to \exp( x^2 \tau/2)$ is a continuous mapping, $ \exp ( Z_n^2 \tau /2 )$ converges to $\exp(Z^2 \tau/2)$ in distribution. 
Moreover, $\{ \exp ( Z_n^2 \tau /2 ) \}$ are uniformly bounded in $L_{1+\epsilon} $ norm for some $\epsilon>0$ and thus uniformly integrable. In particular, 
\begin{align*}
\expect{ \exp ( (1+\epsilon) Z_n^2 \tau /2  ) } = \int_{0}^\infty \prob{ \exp ( (1+\epsilon) Z_n^2 \tau/2 )> t } \diff t  
& = \int_{0}^\infty \prob{Z_n > \sqrt{\frac{2 \ln t }{ (1+\epsilon) \tau }} } \diff t \\
& \overset{(a)}{=} \int_{0}^\infty t^{-\frac{1} {(1+\epsilon) \tau} } \diff t \overset{(b)}{<} \infty,
\end{align*}
where $(a)$ follows from the Hoeffding's inequality $\prob{Z_n  \ge t } \le \exp (-t^2/2)$; $(b)$ holds by choosing $\epsilon$ sufficiently
small such that  $(1+\epsilon) \tau<1$. Hence, $ \mathbb{E} [ \exp ( Z_n^2 \tau /2 ) ] $ converges to $ \mathbb{E}[ \exp(Z^2 \tau/2) ]= \frac{1}{\sqrt{1-\tau} }.$
%Using Lemma 5.4 in \cite{Mossel12} with  \prettyref{lmm:WV}, 
It follows from \prettyref{eq:secondmement4} that when $\tau<1$, as $n \to \infty$,
\begin{align}
\expects{Y_n^2}{\mathbb{P}^\prime_n} \to \frac{ \exp^{-\tau/2-\tau^2/4}  }{ \sqrt{1-\tau}} . \nonumber
\end{align}
%Note that by (\ref{EqTau}),
%\begin{align}
%\sum_{k \ge 3} \sum_{[l]_k} \lambda([l]_k) \eta^2([l]_k) = -\frac{\log(1-\tau)+\tau+\tau^2/2 }{2}. \nonumber
%\end{align}
Hence, in view of \prettyref{eq:expressiontau}, condition 3) of Theorem \ref{ThmSubgraphCond} holds and the second part of Theorem \ref{ThmACER} follows from Theorem \ref{ThmSubgraphCond}.

\section{Conclusion} \label{SectionConclusion}
Our results show that when $\tau<1$ it is fundamentally impossible to give a positively correlated reconstruction; 
when $\tau$ is large enough, the labeling information can be effectively exploited through the suitably weighted graph.
%The simulation result indicates that $\rm{Spectral-Reconstruction}$ leaving out step 1) outputs a type assignment correlated to the true type assignment when $\tau>1$.
An interesting future work is to prove the positive part of Conjecture \ref{Conjecture}.

\section{Acknowledgement}
J. X.\ would like to thank Yudong Chen and  Bruce Hajek for helpful conversations related to this project.
M. L.\ acknowledges the support of the French
Agence Nationale de la Recherche (ANR) under reference
ANR-11-JS02-005-01 (GAP project).
J. X.\ acknowledges the
support of the National Science Foundation under Grant ECCS
10-28464.

%\section{ADDITIONAL PROOFS}

\bibliographystyle{abbrv}
\bibliography{BibCommunityDetection,CommunityDetectionThesis}

\appendix

\section{Special case of Davis-Kahan sin ${\rm \theta}$ Theorem}
The following lemma is Davis-Kahan sin ${\rm \theta}$ theorem \cite{Kahan70} specialized to the rank-$1$ setting. For completeness,
we restate the theorem and provide a proof.
\begin{lemma} \label{lmm:daviskahan}
Let $M=\alpha xx^\top$ and $M'=\beta yy^\top$, with $\alpha,\beta\in
\mathbb{R}$, $\|x\|=\|y\|=1$ and $x^\top y\geq 0$. Then
\begin{eqnarray*}
\|x-y\| \leq \frac{\sqrt{2}}{\max\{|\alpha|,|\beta|\}}\|M-M'\|.
\end{eqnarray*}
Furthermore, if $M'$ is the best rank-$1$ approximation of $\widetilde{M}$, then
\begin{eqnarray*}
\|x-y\| \leq \frac{2\sqrt{2}}{\max\{|\alpha|,|\beta|\}}\|M-\widetilde{M}\|.
\end{eqnarray*}
\end{lemma}
\begin{proof}
First define $\theta\in [0,\pi/2]$ as $x^\top y =\cos \theta \ge 0$. Hence
we have $\|x-y\| = 2\sin \frac{\theta}{2}$. Moreover a simple calculation
shows that $\min_{\gamma\in \mathbb{R}}\|x-\gamma y\|=\sin \theta$ and moreover for
$\theta\in [0,\pi/2]$, we have $\sqrt{2}\sin \frac{\theta}{2}\leq \sin
\theta$. Hence we get $\|x-y\|\leq \sqrt{2}\min_{\gamma}\|x-\gamma
y\|$. Taking $\gamma = \frac{\beta}{\alpha}y^\top x$, then gives
\begin{eqnarray*}
\|x-y\| \leq \sqrt{2}\|x-\frac{\beta}{\alpha}yy^\top x\|= \frac{\sqrt{2}}{|\alpha|}\|(M-M')x\|\leq \frac{\sqrt{2}}{|\alpha|}\|M-M'\|.
\end{eqnarray*}
By symmetry, the first part of the lemma is proved. The second part of the lemma follows from the fact that
\begin{align*}
\|M-M'\| \le \| M -\widetilde{M} \| + \| \widetilde{M} - M' \|  \le 2 \| \widetilde{M} - M \|,
\end{align*}
where the last inequality holds because $M$ is of rank $1$ and $M'$ is the  best rank-$1$ approximation of $\widetilde{M}$.
\end{proof}

\section{Spectrum of Sparse Labeled Stochastic Block Model}

\begin{lemma} \label{lmm:spectrumsparse}
Assume $a\ge b>C_0$ for some sufficiently large constant $C_0$.
There exists some absolute constant
$C$ such that conditional on $\sigma$, 
\begin{align}
\|W^\prime- \mathbb{E}[W | \sigma] \| \leq C \sqrt{a+b}, \quad a.a.s. \nonumber
\end{align}
\end{lemma}
For the special case of  \ER random graph, \ie, $w(\ell)=1$ for all $\ell$ and $a=b$, \prettyref{lmm:spectrumsparse} is proved in \cite{Feige05}.
Our analysis is very similar to that given in \cite{Feige05} with small technical differences due to the edge weights. We provide
a formal proof below for completeness.
\begin{proof}
Define $\mathcal{V}$ be the (random) set of vertices remained and
$\mathcal{V}^c$ denote the set of vertices removed.
For every vertex, its degree is distributed as $\Binom\left(n-1,\frac{a+b}{2n}\right).$  It is shown by \cite{cogls04}[Lemma 39] that there exists
a constant $C_1>0$ such that a.a.s. $|\mathcal{V}^c| \le   n \exp\left( -C_1 (a+b) \right).$
To prove the lemma, it suffices to show $|x^\top (W^\prime- \mathbb{E}[W | \sigma]) x | =O(\sqrt{a+b})$ for all $x$ such that
$\|x\|_2=1$. The proof ideas borrow from \cite{Friedman89,Feige05,Keshavan10} and consists of three steps:
\begin{enumerate}
\item Reduce the problem by proving the same bound for $x$ belonging to a discrete grid.

\item For the discrete grid, bound the contribution of light pairs (defined below) by applying a union bound and a large deviation estimate.

\item  Bound the contribution of heavy pairs using the bounded degree and the discrepancy properties (defined below) .
\end{enumerate}

\subsection{Reduction to a discrete grid}
For any $0<\epsilon<1$, define a grid $\calT_\epsilon$ which approximates the unit sphere $S^{n-1}=\{ x: \l x \| = 1\}$:
\begin{align*}
\calT_\epsilon= \left \{  x \in \left( \frac{\epsilon}{\sqrt{n}} \integers \right)^n: \|x \| \le 1 \right\}.
\end{align*}
For every point $x \in S^{n-1}$, there exists some point $y \in \calT_\epsilon$ such that
$\|x-y\| \le \epsilon$. Therefore, $\calT_\epsilon$ is an $\epsilon$-net of $S^{n-1}$. Moreover, the hypercubes of side length $\epsilon/\sqrt{n}$ centered at the points in $\calT_\epsilon$
are disjoint. On the other hand, all such hypercubes lie in the ball of radius $(1+\epsilon/2)$ centered at the origin.
Since the volume of a unit ball is $ \frac{1+o(1)}{\sqrt{n \pi} } \left( \frac{2\pi}{n}\right)^{n/2}$,
\begin{align}
|\calT_\epsilon| \le \frac{1+o(1)}{\sqrt{n \pi} } \left( \frac{2\pi}{n}\right)^{n/2} \left(1+ \frac{\epsilon}{2}\right)^n \left( \frac{\sqrt{n} }{\epsilon} \right)^n = \exp \left( n  \left[ \log \left( \frac{1}{2} + \frac{1}{\epsilon} \right) + \frac{1}{2} \log (2\pi) +o(1) \right] \right). \label{eq:cardiepsilonnet}
\end{align}
Lemma 5.4 in \cite{vershynin2010nonasym}
implies that
\begin{align*}
\| W'-\expect{W|\sigma} \| = \sup_{x \in S^{n-1}} |x^\top (W'-\expect{W|\sigma}) x| \le (1-2\epsilon)^{-1} \sup_{x \in \calT_\epsilon}
| x^\top (W'-\expect{W|\sigma}) x|.
\end{align*}
Choosing $\epsilon=\frac{1}{4}$, we have $\lnorm{W'-\expect{W|\sigma}}{2} \le 2 \sup_{x \in \calT_{1/4}}
| x^\top (W'-\expect{W|\sigma}) x|$. Hence, it suffices to bound $\sup_{x \in \calT_{1/4} } | x^\top (W'-\expect{W}) x|$.

%Step 1: In particular, define a set of discrete points as follows. Fix $\Delta=1/2$, and let $
%\mathcal{D}=\{ x \in (\frac{\Delta}{\sqrt{n}} \mathbb{Z} )^n: \|x\|_2 \le 1 \}$.

%Claim 2.4 in \cite{Feige05} restated below reduces the problem to proving a similar bound only for pairs of vectors from $\mathcal{D}$.
%\begin{lemma}
%Let $M$ be a $n$ by $n$ matrix. If $|x^\top M y| \le B$ for all $x,y \in \mathcal{D}$, then $x^\top M y \le (1-\Delta)^{-2}B$ for all $x,y$ with $\|x\|_2 \le 1 $ and $\|y\|_2\le 1$.
%\end{lemma}
\subsection{Bounding the contribution of light pairs}
Given an $x \in \calT_{1/4}$, directly applying the concentration inequality to $x^\top (W^\prime- \mathbb{E}[W |\sigma]) x$, such as Bernstein's inequality, does not give the desired result. Define the set of light pairs $L_x \triangleq \{ (u,v): u<v, |x_u x_v| < \frac{\sqrt{a+b}}{n} \}$ and the set of heavy pairs $H_x \triangleq \{ (u,v): u<v\} \setminus L_x$. Observe that
\begin{align}
\sup_{x \in \calT_{1/4}} |x^\top (W^\prime- \mathbb{E}[W |\sigma]) x|\le \sup_{x \in \calT_{1/4}} \bigg| \sum_{(u,v) \in L_x} x_u W^\prime_{uv} x_v- x^\top \mathbb{E}[W |\sigma] x \bigg|+ \sup_{x \in \calT_{1/4}}\bigg| \sum_{(u,v) \in H_x}  x_u W^\prime_{uv} x_v \bigg|. \nonumber
\end{align}
We bound the contribution of heavy pairs separately in the next subsection. Recall that $\mathcal{V}$ denote the set of vertices remained.
Given $\mathcal{V}=V$, define $W^{V}$ by setting to zero the rows and columns of $W$ corresponding to vertices removed, and define the event
\begin{align*}
E(V)= \left\{\sup_{x \in \calT_{1/4} } \bigg|\sum_{(u,v) \in L_x} x_u W^V_{uv} x_v-x^\top \mathbb{E}[W |\sigma] x \bigg| > C \sqrt{a+b}  \right\}.
\end{align*}
Then
\begin{align}
\mathbb{P} \left\{ \sup_{x \in \calT_{1/4} } \bigg|\sum_{(u,v) \in L_x} x_u W'_{uv} x_v-x^\top \mathbb{E}[W |\sigma] x \bigg| > C \sqrt{a+b}  \right\} = \mathbb{P} \{  E( \mathcal{V}) \} \le  2^n \max_{V} \mathbb{P} \{ E(V)\}. \label{EqProbabilityError}
\end{align}
Lemma \ref{LemmaConcentrationLightCouple} below, together with a union bound over all possible points $x \in \calT_{1/4}$ and \prettyref{eq:cardiepsilonnet},
implies that for any  positive constant $C'_2$, there exists a constant $C>0$ such that $\mathbb{P} \{  E(V) \} \le \exp (-C'_2 n)$. In view of (\ref{EqProbabilityError}) and a union bound, we conclude that $\mathbb{P} \{  E(\mathcal{V} ) \}$ is exponentially small by choosing $C$ large enough.

\begin{lemma} \label{LemmaConcentrationLightCouple}
Fix $x \in \calT_{1/4}$ and $V$ to be the set of vertices remaind. Define $W^{V}$ by setting to zero the rows and columns of $W$ corresponding to vertices removed. Let $X= \sum_{(u,v) \in L_x} x_u W^V_{uv} x_v-x^\top \mathbb{E}[W |\sigma] x$. Assume $a>b>C_0$ for some sufficiently large constant $C_0$.
Then $|\mathbb{E}[X] |  \le 2 \sqrt{a+b}$ and for any constant $C_2>0$, there exists some constant $C_3>0$ such that
\begin{align}
\mathbb{P} \left\{ | X- \mathbb{E}[X]|  > C_3 \sqrt{a+b} \right\} \le \exp(- C_2 n ). \nonumber
\end{align}
\end{lemma}
\begin{proof}
Note that
\begin{align*}
\mathbb{E}[X] &=  \sum_{(u,v) \in L_x }  \frac{\alpha + \beta \sigma_u \sigma_v }{n} \1{u,v \in V} x_u x_v-  x^\top \mathbb{E}[W |\sigma] y, \\
&= - \sum_{(u,v) \in H_x } \frac{\alpha + \beta \sigma_u \sigma_v}{n} x_u x_v -  \sum_{(u,v) \in L_x } \frac{\alpha + \beta \sigma_u \sigma_v}{n} (1 -\1{u,v \in V} ) x_u x_v.
\end{align*}
Since $\alpha, \beta \le a+b$, it follows that
\begin{align}
|\mathbb{E}[X]|
%& \le \frac{1}{n}  \sum_{(u,v) \in H } \bigg|  \alpha + \beta \sigma_u \sigma_v \bigg|   |x_u x_v \bigg| + \frac{1}{n}\sum_{(u,v) \notin H } (\alpha + \beta \sigma_u \sigma_v) (\1{u \notin V} + \1{ v \notin V} ) |x_u| |x_v |  \nonumber \\
\le  \frac{a+b}{n}  \sum_{(u,v) \in H_x} | x_u  y_v |  + \frac{a+b}{n}  ( \sum_{u \notin V} |x_u| \sum_{v} |x_v| + \sum_{u} |x_u| \sum_{v \notin V} |x_v| ). \label{eq:boundmean}
%& \overset{(b)}
\end{align}
Notice that $|H_x| \frac{a+b}{n^2} \le  \sum_{(i,j) \in H_x}  x_i^2 x_j^2 \le 1 $. Thus $|H_x| \le \frac{n^2}{a+b}$ and
by Cauchy-Schwartz inequality,
\begin{align*}
\sum_{(u,v) \in H_x} | x_u  x_v | \le |H_x|^{1/2} \left( \sum_{(u,v) \in H_x} x^2_u  x^2_v  \right)^{1/2} \le \frac{n}{\sqrt{a+b}}.
\end{align*}
Again by  Cauchy-Schwartz inequality,
\begin{align*}
\sum_{u \notin V} |x_u| \sum_{v} |x_v| + \sum_{u} |x_u| \sum_{v \notin V} |x_v| \le 2 (n |V^c| )^{1/2} \left ( \sum_{u \in [n], v \notin V} x_u^2 x_v^2 \right) \le 
 2 (n |V^c|)^{1/2} \le 2 n\eexp^{-C_1(a+b)/2} ,
\end{align*}
where the last inequality follows because a.a.s. $|\mathcal{V}^c| \le   n \exp\left( -C_1 (a+b) \right).$
It follows from \prettyref{eq:boundmean} that
\begin{align*}
|\mathbb{E}[X]| \le \sqrt{a+b} +2 (a+b)   e^{-C_1 (a+b)/2}   \le 2 \sqrt{a+b},
\end{align*}
where the last inequality holds when $a \ge C_0$ for a sufficiently large constant $C_0$.

Below we bound $|X-\mathbb{E}[X]|$ using the Bernstein inequality. Define
\begin{align*}
X_{uv}=  W_{uv} x_u x_v \1{(u,v) \in L_x } \1{u,v \in V}.
\end{align*}
Then $X-\mathbb{E}[X]=2 \sum_{u<v} \left( X_{uv} - \mathbb{E}[X_{uv}] \right).$
Note that $|X_{uv} - \mathbb{E}[X_{uv}] | \le \frac{\sqrt{a+b}}{n}$ and
$\var(X_{uv}) \le  x_u^2 x_v^2 \frac{a+b}{n}$. Therefore, $\var(X) \le \frac{a+b}{n} \sum_{u,v} x_u^2 x_v^2 = \frac{a+b}{n}$.
It follows from the Bernstein inequality that for any positive universal constant $C_2>0$,
\begin{align*}
\prob{ \big|X-\mathbb{E}[X] \big| \le \sqrt{2C_2 (a+b) } + \frac{4C_2}{3} \sqrt{a+b}  } \le \eexp^{-C_2n}.
\end{align*}
%Let $c_3= c_\epsilon \triangleq 2 \log \left( \frac{1}{2} + \frac{1}{\epsilon} \right) +  \log (2\pi)$. Then applying the union bound together with \prettyref{eq:cardiepsilonnet}, we get that  with  probability at least $1- \eexp^{-\frac{n}{2}c_\epsilon}$, $|Z_x| \le  2c_\epsilon  \sqrt{np}$ for all $x \in \calT_\epsilon$.
\end{proof}

\subsection{Bounding the contribution of heavy pairs}
 For the set of heavy pairs, since $w(\ell) \in [-1,1]$, it follows that
\begin{align}
\sup_{x \in \calT_{1/4}} | \sum_{(u,v) \in H_x} x_u W^\prime_{uv} x_v |  \le  \sup_{x \in \calT_{1/4}} \sum_{(u,v) \in H}  |x_u y_v| A^\prime_{uv}, \label{eq:heavypairs}
 \end{align}
where $A^\prime$ is defined by setting to zero the rows and columns of $A$ corresponding to vertices removed.
We upper bound \prettyref{eq:heavypairs} by showing that the graph $G'$ with the adjacency matrix given by $A'$ satisfy the following two properties.
%Remark 4.5 in \cite{Keshavan10} shows that there exists a constant $C$ such that $ \sum_{(u,v) \in H}  |x_u y_v| A^\prime_{uv} \le C \sqrt{a+b} $ with probability at larger than $1-n^{-3}$.  Combining Step 1, 2 and 3, the lemma follows.

\begin{definition}[{\bf Bounded degree of order $(d,c_4)$}]
A graph is said to have bounded degree property of order $(d,c_4)$ if every vertex has a degree bounded by $c_4 d$ for some universal constant $c_4>1$.
\end{definition}

\begin{definition}[{\bf Discrepancy of order $(d,c_5,c_6)$}]
A graph is said to have discrepancy property of order $(d,c_5,c_6)$ if for every $S, T \subset[n]$ with $|T|> |S|$, one of the following holds:
\begin{enumerate}
\item $e(S,T) \le c_5 \frac{\eexp d}{n} |S| |T|.$
\item $ e(S,T) \log \left(  \frac{e(S,T) n}{ d |S| |T| }\right) \le c_6 |S| \log \frac{n}{|T|},$
\end{enumerate}
where $e(S,T)$ denotes the set of edges  between vertices in $S$ and vertices in $T$.
\end{definition}

Thanks to removal of edges incident to vertices with degree larger than $\frac{3}{2}\frac{a+b}{2}$, $G'$ satisfy the bound degree property of order $\left(\frac{a+b}{2}, \frac{3}{2}\right)$. In the case with $|T| \ge \frac{n}{\eexp}$, then
\begin{align*}
e(S,T) \le |S| \frac{3(a+b) }{4} \le  \frac{3\eexp (a+b) }{4n} |S| |T|,
\end{align*}
where the first inequality follows from the bounded degree property. Therefore, $G'$ satisfy the discrepancy property with $d=a+b$ and $c_5=3/4$.

In the case with $|T| < \frac{n}{\eexp}$,
let $\widetilde{G}$ denote an \ER random graph with $n$ vertices and edge probability $(a+b)/n$; there exists a coupling such that
if $(u,v) \in E(G)$, then $(u,v) \in E(\widetilde{G})$. It is shown in \cite[Section 2.2.5]{Feige05} that  with probability at least $1-1/n$, $\widetilde{G}$ satisfies the discrepancy property of order $(a+b, c_5, c_6)$ for some constants $c_5$ and
$c_6$. Since removal of edges only decreases $e(S,T)$, $G'$ also satisfies the discrepancy property of order $(a+b, c_5, c_6) $ with probability at least $1-1/n$.

Applying [Corollary 2.11]\cite{Feige05}, we conclude that there exists some constant $C$ such that
\begin{align*}
\mathbb{P} \left\{ \sup_{x \in \calT_{1/4}}  \sum_{(u,v) \in H_x} | x_i x_j| A'_{uv} \le C \sqrt{a+b} \right\} \ge 1- \frac{1}{n}
%\label{eq:boundheavypairsdev}
\end{align*}
The conclusion follows in view of \eqref{EqProbabilityError} and \prettyref{eq:heavypairs}.
\end{proof}
\section{Proof of Lemma \ref{PropCouplingTree}} \label{PfPropCouplingTree}
We introduce some necessary notations for the labeled tree $T$. For a vertex $v \in T$, let $Y_v$ denote the number of children of $v$. Let $Y_v^{=}$ denote the number of children of $v$ with the same type as $v$ and $Y_v^{\neq}=Y_v-Y_v^{=}$. By Poisson splitting property, $Y_v^{=}$ and $Y_v^{\neq}$ are independent Poisson random variables with mean $a/2$ and $b/2$, respectively. Let $Y_v^{\ell}$ denote the number of children of $v$ with the edge connected to $v$ being labeled with $\ell$. Let $Y_v^{=,\ell}$ denote the number of children of $v$ with the same type as $v$ and the edge connected to $v$ being labeled with $\ell$ and $Y_v^{\neq,\ell}= Y_v^{\ell}-Y_v^{=,\ell}$. Then $Y_v^{=,\ell}$ and $Y_v^{\neq,\ell}$ are independent Poisson random variables with mean $(a\mu(\ell)/2)$ and $(b\nu(\ell)/2)$, respectively.

Similarly introduce the corresponding notations for $G_R$. Let $V(G_R)$ denote the set of vertices of $G_R$ and $V_R=V \setminus V(G_R)$. Let $V_R^{+1}$ denote the vertices of type $+1$ in $V_R$ and similarly for $V_R^{-1}$. For a vertex $v \in \partial G_R$, let $X_v$ denote the number of children of $v$ in $V_R$ and $X_v^{=}$ denote the number of children of $v$ in $V_R$ with the same type as $v$. Let $X_v^{\neq}=X_v-X_v^{=}$. Then, $X_v^{=} \sim $ \Binom$(|V_R^{\sigma_v}|,a/n)$ and $X_v^{\neq}\sim $ \Binom$(|V_R^{-\sigma_v}|,b/n)$. Let $X_v^{\ell}$ denote the number of children of $v$ in $V_R$ with edge connected to $v$ being labeled with $\ell$. Let $X_v^{=,\ell}$ denote the number of children of $v$ in $V_R$ with the same type as $v$ and the edge connected to $v$ being labeled with $\ell$ and $X_v^{\neq,\ell}= X_v^{\ell}-X_v^{=,\ell}$. Then, $X_v^{=,\ell} \sim $ \Binom$(|V_R^{\sigma_v}|,a\mu(\ell)/n)$ and $X_v^{\neq, \ell}\sim $ \Binom$(|V_R^{-\sigma_v}|,b\nu(\ell)/n)$.
Note that it is possible to have $u,v \in \partial G_R$ which share the same child in $V_R$ and thus $G_R$ may not be a tree. The goal is to show that such events are rare.

In particular, for any integer $1 \le r \le R$, let $A_r$ denote the event that no vertex in $V_r$ has more than one parent in $G_r$. Let $B_r$ denote the event that there are no edges within $\partial G_r$. Define an event $C_r$ as
\begin{align}
C_r= \{ | \partial G_s | \le 2^s (a+b)^s \log n \text{ for all } 0 \le s \le r \}, \nonumber
\end{align}
which is useful to establish that $V_r$ is large enough so that the binomial distribution is close to Poisson distribution.
Lemma 4.4 and 4.5 in \cite{Mossel12} show that for any $r \le R$,
\begin{align}
\mathbb{P}_n(A_r |C_{r}, \sigma)\ge 1 - O(n^{-3/4}), \nonumber \\
\mathbb{P}_n(B_r |C_{r}, \sigma)\ge 1 - O(n^{-3/4}), \nonumber \\
\mathbb{P}_n (C_{r+1} | C_{r}, \sigma) \ge 1- n^{-\log (4/e) }, \label{Eq:EventABC}
\end{align}
and $|G_{r}| = O(n^{1/8})$ on $C_{r}$.

 We are ready to prove the proposition. Let $V^{+1}$ and $V^{-1}$ denote the set of vertices in $V$ with type $+1$ and $-1$, respectively. Then a.a.s. $| |V^{+1}| - |V^{-1}|| \le n^{3/4}$ in view of \prettyref{eq:clustersizebalance}. Suppose that $(G_r,L_{G_r},\sigma_{G_r})=(T_r, L_{T_r},\sigma_{T_r})$ and $C_{r}$ holds. By (\ref{Eq:EventABC}), the event $A_r$, $B_r$ and $C_{r+1}$ hold simultaneously with probability at least $1-O(n^{-1/8})$ and $|G_{r}| = O(n^{1/8})$. Note that if further $X_v^{=,\ell}=Y_v^{=,\ell}$ and $X_v^{\neq,\ell}=Y_v^{\neq,\ell}$ for every $v \in \partial G_{r}$  and every $\ell \in \mathcal{L}$, then $(G_{r+1},L_{G_{r+1}},\sigma_{G_{r+1}})=(T_{r+1}, L_{T_{r+1}},\sigma_{T_{r+1}})$.

For each $v \in \partial G_r$, $X_v^{=,l} \sim $ \Binom$(|V_r^{\sigma_v}|,a\mu(l)/n)$, and
\begin{align}
  n/2+ n^{3/4} \ge |V^{\sigma_v}| \ge |V_r^{\sigma_v}| \ge |V^{\sigma_v}| - |G_{r}| \ge n/2 - n^{3/4} - O(n^{1/8}). \nonumber
\end{align}
Lemma 4.6 in \cite{Mossel12} bounds total variation distance between binomial and Poisson random variables as
\begin{align}
\| \text{\Binom} \left(m,\frac{c}{n} \right) -\text{\Pois} (c) \|_{\text{TV}} = O\left(\frac{\max \{1, |m-n| \} }{n} \right). \nonumber
\end{align}
Therefore, for any fixed $v \in \partial G_r$ and $\ell \in \mathcal{L}$, $X_v^{=,\ell}$ can be coupled with $Y_v^{=,\ell}$ such that $\mathbb{P}\{ X_v^{=,\ell} \neq  Y_v^{=,\ell}\} = O(n^{-1/4})$ and similarly for $X_v^{\neq, \ell}$. Since $|\partial G_r|= O(n^{1/8})$ and $\mathcal{L}$ is a finite set, the union bound concludes that  $X_v^{=,\ell}=Y_v^{=,\ell}$ and $X_v^{\neq,\ell}=Y_v^{\neq,\ell}$ for every $v \in \partial G_{r}$  and every $\ell \in \mathcal{L}$ with probability at least $1-O(n^{-1/8})$,
Therefore
\begin{align}
 \mathbb{P} \left\{ (G_{r+1},L_{G_{r+1}},\sigma_{G_{r+1}})=(T_{r+1}, L_{T_{r+1}},\sigma_{T_{r+1}}), C_{r+1} \big| (G_r,L_{G_r},\sigma_{G_r} ) =(T_r, L_{T_r},\sigma_{T_r}), C_{r} \right\} \ge 1-O(n^{\frac{1}{8}}). \nonumber
\end{align}
By definition of condition probability,
\begin{align}
 &\mathbb{P} \{ (G_{r+1},L_{G_{r+1}},\sigma_{G_{r+1}})=(T_{r+1}, L_{T_{r+1}},\sigma_{T_{r+1}}), C_{r+1} \}  \nonumber \\
 &\ge \left( 1-O(n^{-1/8}) \right)
 \mathbb{P} \{ (G_{r},L_{G_{r}},\sigma_{G_{r}})=(T_{r}, L_{T_{r}},\sigma_{T_{r}}), C_{r} \}. \label{eq:recursive}
\end{align}
Since $\mathbb{P}(C_0)=1$, and $G_R$ and $T_R$ starts at the same root $\rho$, the proposition follows by recursively applying \prettyref{eq:recursive}.
%by a union bound over $r=1, \ldots, R$.

\section{Proof of Lemma \ref{LemmaNumCycles}} \label{PfLemmaNumCycles}
First consider the graph distribution $\mathbb{P}_n^\prime$. By the method of moments (see Theorem 6.10 \cite{Janson11}), it suffices to show that under $\mathbb{P}_n^\prime$,
\begin{align}
\mathbb{E} \left[\prod_{k=1}^m \prod_{[\ell]_k} (X_n([\ell]_k))_{j([\ell]_k)}  \right] \to \prod_{k=1}^m \prod_{[\ell]_k} (\lambda([\ell]_k))^{j([\ell]_k)}, \label{EqcyclesGnp}
\end{align}
for all possible non-negative integers $\{ j([\ell]_k) \}$. We first show that for any fixed $[\ell]_k$, $\mathbb{E}[X_n([\ell]_k)] \to \lambda([\ell]_k)$.

Let $v_0,\ldots, v_{k-1}$ be $k$ distinct vertices among $n$ vertices. Let $I$ be the indicator that $v_0,\ldots, v_{k-1}$ is a $k$-cycle with labels $[\ell]_k$. Then,
\begin{align}
\mathbb{E}[I] =\prod_{s=1}^k \frac{a \mu(\ell_s) +b \nu(\ell_s)}{2n}.
\end{align}
By the linearity of expectation,
\begin{align}
\mathbb{E}[X_n([\ell]_k)] \overset{(a)}{=} \binom{n}{k} \frac{(k-1)!}{2} \mathbb{E}[I] = \binom{n}{k} \frac{(k-1)!}{2} \prod_{s=1}^k \frac{a \mu(\ell_s) +b \nu(\ell_s)}{2n}, \nonumber
\end{align}
where $(a)$ holds because there are $\binom{n}{k}$ different choices of $v_0,\ldots, v_{k-1}$ and $k!$ different permutations of them; each cycle corresponds to $2k$ different permutations.
Therefore, $\mathbb{E}[X_n([\ell]_k)] \to \lambda([\ell]_k)$ as $n \to \infty$ as long as $k=o (\sqrt{n})$.

Then, we argue that $\mathbb{E}[(X_n([\ell]_k))_j] \to (\lambda([\ell]_k))^j$. Note that $(X_n([\ell]_k))_j$ is the number of ordered $j$-tuples of $k$-cycles with labels $[\ell]_k$ in $G$. Divide these $j$-tuples into two sets: $A$ is the set of $j$-tuples for which all of the $k$-cycles are disjoint, and $B$ is the set of the rest of the $j$-tuples.

Take $(C_1,C_2,\ldots,C_j) \in A$. Since $C_i$'s are disjoint, they appear independently. By the previous argument, it follows that the cycles $C_1,\ldots,C_j$ are all present in $G$ with probability
\begin{align}
\prod_{i=1}^j \prod_{s=1}^k \frac{a \mu(\ell_s) +b \nu(\ell_s)}{2n}. \nonumber
\end{align}
Since there are $\binom{n}{kj}\frac{(kj)!}{(2k)^j}$ elements in $A$, the expected number of vertex-disjoint $j$-tuples of $k$-cycles with $[\ell]_k$ is
\begin{align}
\binom{n}{kj}\frac{(kj)!}{(2k)^j} \prod_{i=1}^j \prod_{s=1}^k \frac{a \mu(\ell_s) +b \nu(\ell_s)}{2n} \to (\lambda([\ell]_k))^j. \nonumber
\end{align}
Let $\tilde{I}$ be the number of non-vertex-disjoint $j$-tuples. Then the distribution of $\tilde{I}$ is stochastically dominated by the distribution of $\tilde{I}$ under an Erd\H{o}s-R\'enyi random graph $\mathcal{G}(n,\frac{\max \{a, b\} }{n})$. It is shown by \cite{bollobas01}[Corollary 4.4] that if $k=O(\log^{1/4} n)$, then $\mathbb{E}[\tilde{I}] \to 0$ for any $\mathcal{G}(n,c/n)$ with constant $c$ . Hence, $\mathbb{E}[\tilde{I}] \to 0$ under $\mathbb{P}_n^\prime$.

Finally, note that the same argument applies to any joint factorial moment corresponding to cycles with different lengths and labels. Thus equation (\ref{EqcyclesGnp}) follows.

Next consider the graph distribution $\mathbb{P}_n$. It suffices to show that under $\mathbb{P}_n$
\begin{align}
\mathbb{E} \left[\prod_{k=1}^m \prod_{[\ell]_k} (X_n([\ell]_k))_{j([\ell]_k)}  \right] \to \prod_{k=1}^m \prod_{[\ell]_k} (\xi([\ell]_k))^{j([\ell]_k)} . \label{EqcyclesSBM}
\end{align}
We claim that for any fixed $[\ell]_k$, $\mathbb{E}[X_n([l]_k)] \to \xi([\ell]_k)$. Let $v_0,\ldots, v_{k-1}$ be $k$ distinct vertices among $n$ vertices. Let $I$ be the indicator that $v_0,\ldots, v_{k-1}$ is a $k$-cycle with labels $[\ell]_k$ and $v_0$ being the vertex with the minimum index. Let $v_k=v_0$, then $\mathbb{E}[X_n([\ell]_k)] = \binom{n}{k} \frac{(k-1)!}{2} \mathbb{E}[I]$ and
\begin{align}
\mathbb{E}[I | \sigma_{v_0}, \ldots, \sigma_{v_{k-1}} ]= n^{-k} \prod_{ \substack{ 1 \le i \le k \\ \sigma_{v_{i-1}}=\sigma_{v_i} } } a \mu(\ell_i) \prod_{ \substack{ 1 \le i \le k \\ \sigma_{v_{i-1}} \neq \sigma_{v_i} } } b \nu(\ell_i). \nonumber
\end{align}
Notice that there are always even number of $i$ such that $\sigma_{v_{i-1}} \neq \sigma_{v_i}$. Thus,
\begin{align}
\mathbb{E}[I]= \mathbb{E}_{\sigma}[ \mathbb{E}[I | \sigma  ] ]=(2n)^{-k} \left( \prod_{i=1}^k (a\mu(\ell_i)+ b\nu(\ell_i)) + \prod_{i=1}^k (a\mu(\ell_i)-b\nu(\ell_i)) \right). \nonumber
\end{align}
Therefore, $\mathbb{E}[X_n([\ell]_k)] \to \xi([\ell]_k)$ and by the same
argument as before, equation (\ref{EqcyclesSBM}) holds.
\section{Bernstein Inequality}
\begin{theorem} \label{thm:Bernstein}
Let $X_1, \ldots, X_n$ be independent random variables such that $| X_i | \le M$ almost surely. Let $\sigma_i^2=\var(X_i)$ and $\sigma^2= \sum_{i=1}^n \sigma_i^2$, then
\begin{align}
\mathbb{P} \left\{  \sum_{i=1}^n X_i \ge t \right\} \le \exp \left (  \frac{-t^2}{ 2 \sigma^2 + \frac{2}{3} M t }\right). \nonumber
\end{align}
It follows then
\begin{align}
\mathbb{P} \left\{ \sum_{i=1}^n X_i \ge \sqrt{2 \sigma^2 u} + \frac{2M u }{3}  \right\} \le e^{-u}. \nonumber
\end{align}
\end{theorem}

\end{document}